\documentclass{article}

\usepackage[preprint]{neurips_2026}

\usepackage[utf8]{inputenc} 
\usepackage[T1]{fontenc}    
\usepackage[hidelinks]{hyperref} 
\usepackage{url}            
\usepackage{booktabs}       
\usepackage{amsfonts}       
\usepackage{amssymb}
\usepackage{amsmath}
\usepackage{nicefrac}       
\usepackage{microtype}      
\usepackage{xcolor}         
\usepackage{graphicx}
\usepackage{subcaption}
\usepackage{multirow}
\usepackage{wrapfig}
\usepackage{amsthm}

\newtheorem{theorem}{Theorem}
\newtheorem{corollary}{Corollary}
\newtheorem{proposition}{Proposition}
\newtheorem{definition}{Definition}
\newtheorem{lemma}{Lemma}

\newcommand{\vz}{\mathbf{z}}
\newcommand{\vx}{\mathbf{x}}

\newcommand{\real}{\mathbb{R}}
\newcommand{\E}{\mathbb{E}}
\newcommand{\KL}{\mathrm{KL}}
\newcommand{\cL}{\mathcal{L}}

\title{MOSAIC: Module Discovery via Sparse Additive\\Identifiable Causal Learning for Scientific Time Series}

\author{%
  Shicheng Fan$^1$ \quad Nour Elhendawy$^2$ \quad Jianle Sun$^3$ \quad Ke Fang$^1$ \\
  \bfseries Kun Zhang$^{3,4}$ \quad Yihang Wang$^2$ \quad Lu Cheng$^1$ \\[0.5ex]
  \normalfont $^1$University of Illinois Chicago \quad $^2$Case Western Reserve University \\
  \normalfont $^3$Carnegie Mellon University \quad $^4$MBZUAI \\
  \normalfont\small \texttt{\{sfan25,kfang10,lucheng\}@uic.edu} \quad \texttt{\{nae47,yihang.wang4\}@case.edu} \\
  \normalfont\small \texttt{jianles@andrew.cmu.edu} \quad \texttt{kunz1@cmu.edu}
}

\begin{document}

\maketitle

\footnotetext[1]{Open source code: \url{https://github.com/shichengf/mosaic}}

\begin{abstract}

Causal representation learning (CRL) seeks to recover latent variables with identifiability guarantees, typically up to permutation and component-wise reparameterization under appropriate assumptions. However, identifiability does not imply interpretability: latent semantics are typically assigned post hoc by alignment with known ground-truth factors. This limitation is particularly acute in scientific time series, where underlying mechanisms are unknown and discovering interpretable structure is a primary goal. In contrast, scientific observations (such as residue-pair distances, climate indices, or process sensors) are inherently semantic, as they correspond to named physical quantities. This raises a key question: \textit{can the interpretability of observations be transferred to the identifiable latent space?}
We propose \textbf{MOSAIC} (\textbf{Mo}dule discovery via \textbf{S}parse \textbf{A}dditive \textbf{I}dentifiable \textbf{C}ausal learning), a sparse temporal VAE that integrates temporal CRL identifiability with support recovery over observed variables. MOSAIC identifies latent variables via regime-conditioned temporal variation, and recovers for each latent a sparse set of associated observations through an additive decoder, yielding module-level interpretability. We show that ANOVA main-effect supports are identifiable under general smooth mixing functions, with finite-sample recovery guarantees for a group-lasso estimator and empirical evaluation of MOSAIC's entropy regularizer. Empirically, MOSAIC recovers domain-consistent variable groups across RNA molecular dynamics, solar wind, ENSO climate, the Tennessee Eastman process, and a synthetic tokamak benchmark, enabling interpretable discovery of latent mechanisms in scientific time series.
\end{abstract}

\section{Introduction}
\label{sec:intro}
Causal representation learning (CRL)~\citep{scholkopf2021toward, zhang2024causal} aims to recover latent variables uniquely up to trivial transformations, a property known as identifiability. In particular, temporal CRL under regime-conditioned dynamics~\citep{hyvarinen2016unsupervised, yao2022temporally} achieves identifiability up to permutation and component-wise transformations. However, identifiability alone does not guarantee scientific interpretability. In standard CRL settings, the meaning of learned latent variables is typically assigned post hoc by aligning them with known ground-truth factors, rather than being directly inferred by the model.

Time series in scientific domains present a complementary yet contrasting setting. Their observed variables are often named physical quantities (e.g., residue-pair distances in protein systems) and therefore inherently carry domain semantics. In contrast, the latents governing these observations correspond to underlying physical mechanisms that are typically unknown, and uncovering them is essential to scientific analysis. Existing scientific representation learning approaches (e.g.,\citep{Wang_2023}) leverage the semantic structure of observed variables, but do not provide identifiability guarantees for the uncovered latent mechanisms. As a result, scientific data offer interpretable observations but lack identifiable latents, whereas CRL provides identifiable latents without inherent scientific interpretation.

This suggests introducing CRL into scientific time series, but direct application of existing temporal CRL methods remains insufficient. Approaches such as TDRL~\citep{yao2022temporally, yao2021learning, song2024causal, fan2026trace} identify latent variables, but rely on dense neural decoders that do not reveal which named observations each latent controls. As a result, the semantics of scientific observations are not transferred to the latent space. 
To address this gap, we propose \textbf{MOSAIC} (\textbf{Mo}dule discovery via \textbf{S}parse \textbf{A}dditive \textbf{I}dentifiable \textbf{C}ausal learning), a sparse temporal VAE that integrates temporal CRL with support recovery over named observations. MOSAIC first identifies latent variables via regime-conditioned temporal variation, and then recovers, for each latent, a sparse set of associated observations through an additive decoder. This transforms otherwise anonymous but identifiable latents into physically interpretable latent mechanisms.

The key idea is that MOSAIC uses an additive decoder that structurally separates each latent's contribution to the observations: each latent controls its own small network mapping to the observation space, so the set of observations it influences is well-defined and readable. We show that even when the true mixing function is non-additive, this additive fit recovers the correct per-latent influence pattern (Proposition~\ref{prop:no_fp}), so the recovered supports are not artifacts of the modeling restriction. To scale to high-dimensional scientific data, MOSAIC also introduces a parallel transition prior removing the sequential bottleneck in prior temporal CRL~\citep{yao2022temporally,song2024causal,li2024identification}.

\vspace{-3pt}
\begin{figure}[t]
\centering
\includegraphics[width=1\textwidth]{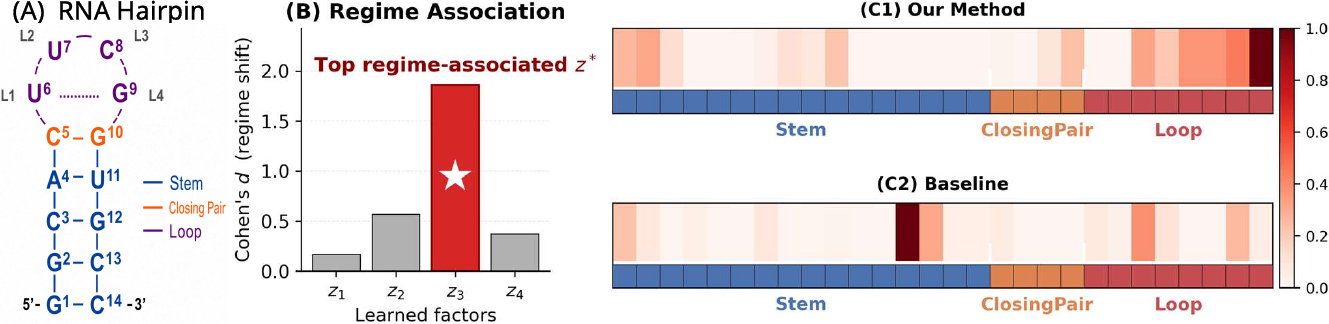}
\caption{\textbf{RNA hairpin example.} \textbf{(A)} The cUUCGg RNA hairpin; colors mark three structural regions (Sterm, Closing Pair, Loop). \textbf{(B)} Regime-association score (Cohen's $d$) over four learned latents; $z_3$ is selected. \textbf{(C1, C2)} Influence column of $z_3$ under MOSAIC (C1) and TDRL~\citep{yao2022temporally} (C2); colored stripes match the regions in (A).}
\label{fig:rna}
\vspace{-15pt}
\end{figure}

Figure~\ref{fig:rna} illustrates this on an RNA molecule that folds into a hairpin shape. Its 14 residues group into three structural regions visible in Figure~\ref{fig:rna}A: Stem (paired base), Closing Pair (hinge), and Loop (flexible tip). Each of the $D{=}28$ observed variables measures the distance behavior of one residue, so variable names carry physical meaning. MOSAIC ranks four learned latents by regime-association score (Figure~\ref{fig:rna}B) and selects $z_3$, which is anonymous until we inspect its influence column. Under MOSAIC (Figure~\ref{fig:rna}C1), influence concentrates on Loop residues, interpreting $z_3$ as tracking loop flexibility. Under TDRL (Figure~\ref{fig:rna}C2), influence spreads across all groups, leaving $z_3$ uninterpretable. The key point is general: an identifiable latent must be anchored to named observed variables to yield a scientific interpretation.

Our contributions are fourfold. (1)~\textbf{Problem formulation.} We formalize scientific time-series interpretability as recovering identifiable latent variables, their observed-variable supports, and the physical interpretation of regime-associated factors. (2)~\textbf{Theory.} We prove population recovery of ANOVA main-effect supports, module and regime-association identifiability at $\rho=0$, and a finite-sample bound for a group-lasso variant; for entropy regularization we give concentration and top-$k$ statements. (3)~\textbf{Architecture.} We introduce a sparse temporal VAE with two-stage curriculum and parallel transition prior. (4)~\textbf{Experiments.} MOSAIC recovers localized modules on a controlled synthetic benchmark, a UUCG-loop correlate in RNA dynamics (Figure~\ref{fig:rna}), and regime-associated factors on OMNI solar wind, ENSO climate, TEP, and a tokamak-inspired benchmark.

\section{Related Work}
\label{sec:related}
\textbf{Temporal causal representation learning.} Nonlinear ICA is unidentifiable without auxiliary information~\citep{hyvarinen1999nonlinear, locatello2019challenging}, and a now-mature literature restores identifiability through temporal structure~\citep{hyvarinen2016unsupervised, khemakhem2020variational, yao2022temporally, yao2021learning, brehmer2022weakly}, unknown nonstationarity~\citep{song2023temporally, song2024causal}, instantaneous dependence~\citep{li2024identification}, continuous mechanism evolution~\citep{fan2026trace}, or mechanism and structural sparsity in the latent space~\citep{lachapelle2022disentanglement, lachapelle2024nonparametric, zheng2022identifiability, zheng2023generalizing}. All of these methods identify latent variables while keeping the latent-to-observed mapping flat and unstructured; none produces a module partition over named observed variables. Our experimental CRL baselines (TDRL~\citep{yao2022temporally}, iVAE~\citep{khemakhem2020variational}, SlowVAE \citep{klindt2020towards}, $\beta$-VAE \citep{higgins2017betavae}, and CtrlNS~\citep{song2024causal}) span the most relevant settings: given clean regimes (TDRL, iVAE), weak temporal supervision (SlowVAE, $\beta$-VAE), and automatic regime discovery (CtrlNS); all exhibit the same support-recovery limitation in our experiments.

\textbf{Causal discovery for scientific time series.} Observed-space methods (Granger causality~\citep{granger1969investigating}, transfer entropy~\citep{schreiber2000measuring}, convergent cross mapping~\citep{sugihara2012detecting}, PCMCI~\citep{runge2019pcmci}; survey in~\citep{gong2024causal}) recover pairwise dependencies between named variables, and LPCMCI~\citep{gerhardus2020high} detects latent confounders without naming them; recent work recovers latent causal graphs from spatiotemporal data~\citep{wang2025spacy} but does not solve the latent-to-observed support problem. These methods are complementary to MOSAIC: pairwise causal-discovery tools can operate within or between the modules MOSAIC identifies at reduced dimensionality.

\textbf{Scientific representation learning and sparse statistical methods.} VAMPnets~\citep{mardt2018vampnets}, SPIB~\citep{wang2021state}, and TICA~\citep{perez2013tica} produce useful low-dimensional summaries of scientific time series, but without identifiability, without explicit latent-to-observed support, and without a regime-association interpretation protocol. SPIB is our experimental baseline because it explicitly incorporates regime labels into its learned representation; VAMPnets and TICA target slow-mode kinetics rather than regime contrast. Sparse PCA~\citep{zou2006sparse} and ICA~\citep{hyvarinen2000independent} yield readable sparse loadings under near-linear mixing, but a single linear projection cannot fit both saturating and expanding nonlinear responses from the same factor, and they lack temporal or regime-conditioned structure.
\section{Problem Formulation}
\label{sec:problem}


We work in a setting motivated by time series in scientific domains, where each observed variable is a scientific quantity carrying domain semantics. At each time step, we observe $\vx_t \in \real^{D}$, where every $x_i$ corresponds to one such named variable. This semantic structure on $\vx_t$ enables interpretation of recovered latents through their observed-variable supports.

Let $\vx_t$ be generated by $n$ latent factors $\vz_t \in \real^n$ through a smooth, injective mixing function $g(\cdot)$,
\begin{equation}
    \vx_t = g(\vz_t) + \boldsymbol{\epsilon}_t, \qquad g: \real^n \to \real^{D}, \qquad \boldsymbol{\epsilon}_t \sim \mathcal{N}(\mathbf{0}, \sigma_\epsilon^2 \mathbf{I}),
    \label{eq:mixing}
\end{equation}
with no additivity restriction on $g(\cdot)$. The latent modules follow a regime-dependent finite-lag transition model
\begin{equation}
    z_{t,j} \mid \vz_{t-1:t-L}, c_t \sim p_j(z_{t,j} \mid \vz_{t-1:t-L}, c_t),
    \label{eq:transition}
\end{equation}
where $c_t \in \{0, 1\}$ indicates the regime (e.g., folded vs.\ unfolded in RNA dynamics). Here, ``causal'' follows temporal CRL, with $c_t$ serving as an observed conditioning variable rather than an intervention. Modules whose transition dynamics differ across regimes are \emph{regime-varying}; those with invariant dynamics are \emph{regime-invariant}.

\textbf{Main-effect supports as identifiable targets.} Because $g(\cdot)$ need not be additive, a single variable $x_i$ may depend on a latent factor $z_j$ in a way that is irreducibly entangled with other factors. To define a recoverable target, we work under a product reference measure $\mu(\vz)=\prod_{j=1}^{n}\mu_j(z_j)$ over normalized latent factors, the standard setting for Hoeffding ANOVA~\citep{hoeffding1992class}. Under this reference measure, any square-integrable $g$ admits the decomposition
\begin{equation}
    g(\vz) = g_0 + \underbrace{\sum_{j=1}^{n} g_j(z_j)}_{\text{main effects}} + \underbrace{\textstyle\sum_{j<k} g_{jk}(z_j, z_k) + \cdots}_{r(\vz)\text{: high-order interactions}},
    \label{eq:anova}
\end{equation}
where $g_0 = \E_{\mu}[g(\vz)]$, each main effect $g_j: \real \to \real^D$ captures the marginal contribution of $z_j$ (with per-coordinate components $g_j^{(i)}: \real \to \real$ for $i \in [D]$), and $r(\vz)$ collects all higher-order interactions. The expansion is $L^2(\mu)$-orthogonal and uniquely determined by the centering constraints $\E_{\mu_j}[g_j(z_j)]=0$. MOSAIC's decoder targets the main effects $\{g_j\}$, a modeling choice analogous to Sparse PCA targeting the best linear projection. Section~\ref{sec:identifiability} shows the population-optimal additive fit recovers each $g_j$ exactly.

For analysis, we quantify how much of the signal lives in interactions through the \emph{interaction ratio}
\begin{equation}
\rho \;=\; \E_{\mu}\|r(\vz)\|^2 \big/ \E_{\mu}\|g(\vz) - g_0\|^2 \;\in\; [0, 1],
\label{eq:rho}
\end{equation}
with $\rho = 0$ recovering exactly additive mixing. This quantity appears in the finite-sample bound of Section~\ref{sec:identifiability}. Unless otherwise stated, population risks and variances in Section~\ref{sec:identifiability} are evaluated under this product reference measure and the independent observation-noise law.

\begin{definition}[Main-Effect Support Structure]
\label{def:support}
The main-effect support of the $j$-th component is $\mathcal{S}_j = \{i \in [D] : g_j^{(i)} \text{ is non-constant}\}$. When the $\{\mathcal{S}_j\}$ are mutually disjoint, they define a \emph{module partition} of the observed variables, assigning each named $x_i$ to a unique latent factor $z_j$. At $\rho = 0$ these coincide with the full supports of $g$; when $\rho > 0$, variables whose dependence on $z_j$ lives entirely in higher-order terms lie outside $\mathcal{S}_j$.
\end{definition}

\textbf{Recovery goals.} From $\{(\vx_t, c_t)\}$ we aim to recover (i) the latent factors $\vz_t$ up to permutation and component-wise invertible transformation, (ii) the main-effect support $\mathcal{S}_j$ of each recovered latent factor, and (iii) the identity of the latent factor most associated with the regime contrast.

\section{Identifiability Theory}
\label{sec:identifiability}

This section establishes that the recovery goals of Section~\ref{sec:problem} are achievable in the appropriate sense: main-effect supports are identifiable for any smooth mixing function (Proposition~\ref{prop:no_fp}), and at $\rho=0$ the entire module partition together with the regime-associated identity is exactly recoverable (Theorem~\ref{thm:module} and Corollary~\ref{cor:driver}). For finite samples, we give a $O(N^{-4/5})$ estimation rate for the main effects (Proposition~\ref{prop:finite_sample_main_effect}) and a per-channel sample-complexity bound for support recovery (Theorem~\ref{thm:finite}).

\textbf{Terminology.} We use \emph{regime-discriminative} and \emph{regime-associated} interchangeably for the learned factor selected by the regime-discrepancy score, and \emph{regime-varying} for the ground-truth factor whose dynamics differ across regimes.

\textbf{Assumptions.}
(A1) Non-degenerate main effects, $|\mathcal{S}_j|\geq 1$ for all $j$. 
(A1$'$) Disjoint supports, $\mathcal{S}_j \cap \mathcal{S}_k = \emptyset$ for $j \neq k$, used when claiming a module partition or sparse-additive support recovery. 
(A2) Temporal sufficient variability~\citep{yao2022temporally}: the lag-vector cross-derivative matrix has full row rank, providing the auxiliary variation that identifies the latent factors up to permutation and component-wise reparametrization.

Proofs and the shared-support assignment rule (covering $|\mathcal{J}_i| \geq 2$) are in Appendix~\ref{app:theory}.

\subsection{Population-Level Identifiability}

\begin{proposition}[Main-Effect Identifiability Without False Positives]
\label{prop:no_fp}
For any smooth and injective $g$, the centered additive fit $\hat g(\vz) = \sum_j \hat f_j(z_j) + \hat{\mathbf{b}}$ minimizing the $L^2(\mu)$ population risk $\E_{\mu}\|\vx - \hat g(\vz)\|^2$ subject to the centering constraints $\E_{\mu_j}[\hat f_j(z_j)] = 0$ for all $j$, satisfies $\hat f_j = g_j$ almost surely for all $j$. Hence $\hat{\mathcal{S}}_j = \mathcal{S}_j \subseteq \mathcal{S}_j^{\mathrm{full}}$ is independent of $\rho$, where $\mathcal{S}_j^{\mathrm{full}}$ denotes the full support of $g$ in $z_j$.
\end{proposition}

\textbf{Remark.}
Proposition~\ref{prop:no_fp} is stated under the product reference measure $\mu = \prod_j \mu_j$. In practice, Stage~2 optimizes under the empirical encoder distribution, which approximates a product measure to the extent that Stage~1 identifiability succeeds. The gap between the two measures is absorbed by the interaction term $r(\vz)$ and does not alter the main-effect targets $\{g_j\}$; see Appendix~\ref{app:product_measure_gap} for details.

\begin{theorem}[Module Structure Identifiability at $\rho = 0$]
\label{thm:module}
If two models $(g, p)$ and $(\tilde g, \tilde p)$ both satisfy A1 and A2 with $\rho = 0$ and induce the same observational distribution $p(x_{1:T}, c_{1:T})$, then there exists a permutation $\pi$ such that $\tilde{\mathcal{S}}_k = \mathcal{S}_{\pi(k)}$ for all $k$. If additionally A1$'$ holds, then the $\{\mathcal{S}_k\}$ form a module partition and this partition is identifiable.
\end{theorem}
At $\rho=0$, Theorem~\ref{thm:module} sharpens Proposition~\ref{prop:no_fp}: the entire module partition is identifiable, not just individual supports; when supports overlap ($|\mathcal{J}_i| \geq 2$), strict partition breaks down but support recovery remains sound:

\begin{proposition}[Shared-Support Assignment Under Overlap]
\label{prop:partial_overlap}
Let $g$ satisfy $\rho = 0$ and define $\mathcal{J}_i=\{m\in[n]: i\in\mathcal{S}_m\}$. If $|\mathcal{J}_i|\geq2$, then the population additive fit satisfies $\hat f_m^{(i)} = g_m^{(i)}$ for every $m\in\mathcal{J}_i$, and any hard-assignment rule $\arg\max_m A_{im}$ assigns $i$ to a factor in $\mathcal{J}_i$ without false positives.
\end{proposition}

\begin{definition}[Regime-Discriminative Latent]
\label{def:driver}
Let $\Delta_j=\mathrm{KL}\bigl(p(z_j \mid c{=}0) \,\|\, p(z_j \mid c{=}1)\bigr)$. The set of regime-discriminative latents is $\arg\max_j \Delta_j$; when this maximizer is unique, we denote it $j^*$.
\end{definition}

\begin{corollary}[Regime-Association Identifiability]
\label{cor:driver}
Under the conditions of Theorem~\ref{thm:module} with a unique maximizer, $\tilde j^* = \pi^{-1}(j^*)$ by KL invariance under invertible maps. Without uniqueness, the entire maximizer set is mapped by $\pi^{-1}$.
\end{corollary}

The theory uses KL for invariance under component-wise invertible reparameterizations; in experiments, we use Cohen's $d$ as a stable finite-sample proxy for mean shifts, with this substitution validated empirically in Appendix~\ref{app:cohen_kl}.

\subsection{Finite-Sample Behavior}
Estimation efficiency is an active concern in CRL. Our population-level results above hold for any smooth estimator, while finite-sample guarantees are estimator-specific: for the classical group-lasso sparse additive estimator, we record a per-channel support-recovery bound in Appendix~\ref{app:finite_sample} (Theorem~\ref{thm:finite}), where $D$ channels share the same $N$ design points so the sample complexity is $D$-free. MOSAIC instead uses entropy regularization, for which finite-sample guarantees remain open, and we evaluate its performance empirically. We sweep $N$ relative to $D$ and interaction strength $\rho$ on the synthetic benchmark (Appendices~\ref{app:synthetic:nd},~\ref{app:synthetic:interactions}) and calibrate $\hat\rho$ on real data (Appendix~\ref{app:rho_calibration}); Section~\ref{sec:cross_domain} discusses how to interpret these diagnostics on real data.

\vspace{-8pt}
\begin{figure}[t]
\centering
\includegraphics[width=.85\textwidth]{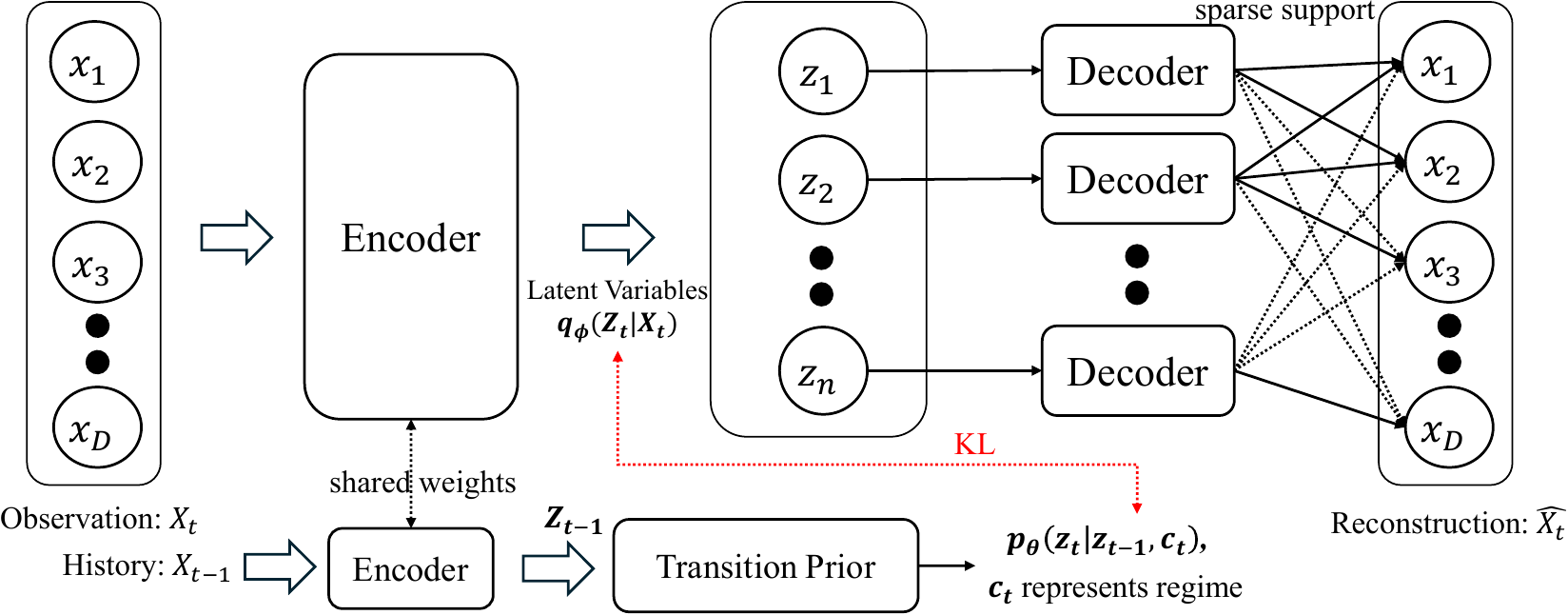}
\caption{MOSAIC architecture. The encoder produces $\vz_t$, which feeds per-factor additive decoders whose aggregated influence profiles form the support matrix $\mathbf{A}$. A regime-conditioned transition prior supplies the temporal identifiability signal via the posterior-prior KL. Stage~1 trains with a dense decoder; Stage~2 freezes the encoder and replaces it with the additive sparse head.}
\label{fig:architecture}
\vspace{-8pt}
\end{figure}

\section{Method}
\label{sec:method}
MOSAIC realizes the theoretical ordering of Section~\ref{sec:identifiability} as a two-stage training curriculum (Figure~\ref{fig:architecture}): Stage~1 identifies latent factors with a temporal prior, and Stage~2 freezes the encoder and recovers each latent's main-effect support through an additive sparse decoder. 

\textbf{Stage~1: Identifiable latent learning.}
A VAE~\citep{khemakhem2020variational} with encoder $q_\phi(\vz_t \mid \vx_t)$ and a dense MLP decoder is trained jointly with a regime-conditioned temporal prior factorized across dimensions,
\begin{equation}
    p(\vz_t \mid \vz_{t-L:t-1}, c_t) = \prod_{j=1}^{n} p_j(z_{t,j} \mid \vz_{t-L:t-1}, c_t),
    \label{eq:prior_factorized}
\end{equation}
where each $p_j$ is a Laplace density on the output of a per-dimension transition MLP receiving past $L$ lag vectors and current factor $z_{t,j}$. The factorized form supplies auxiliary variation for temporal identifiability~\citep{hyvarinen2016unsupervised, yao2022temporally}. The Stage~1 loss combines reconstruction, a standard-Gaussian KL, and the temporal-prior KL; the full form is in Appendix~\ref{app:architecture}.

\textbf{Stage~2: Sparse support recovery.}
With the encoder and prior frozen, the dense decoder is replaced by an additive head $\hat{\vx} = \sum_j f_j(z_j) + \mathbf{b}$, where each $f_j$ is a small MLP estimating the ANOVA main effect $g_j$. After standardizing each latent factor, we form the \emph{functional influence matrix} $\mathbf{A} \in \real^{D \times n}$ via the bias-invariant finite contrast $A_{ij} = |f_j(+1)_i - f_j(-1)_i|$, which cancels the shared bias $\mathbf{b}$ and gives equivalent supports for non-symmetric main effects; symmetric responses (e.g., $z_j^2$) require a variance- or range-based score (Appendix~\ref{app:theory}). To enforce module structure, we apply an entropy penalty on each active factor's normalized influence column:
\begin{equation}
    \cL_{\text{sparse}} = \frac{1}{|\mathcal{J}_{\text{alive}}|} \sum_{j \in \mathcal{J}_{\text{alive}}} H\!\left(A_{:,j} \big/ \textstyle\sum_{i} A_{ij}\right),
    \label{eq:entropy_sparsity}
\end{equation}
where $H(\cdot)$ is the Shannon entropy and $\mathcal{J}_{\text{alive}}$ masks near-dead factors. Low-entropy columns realize the supports $\{\mathcal{S}_j\}$ of Theorem~\ref{thm:module}: each $z_j$ concentrates its influence on a small subset of observed variables. The Stage~2 objective adds $\lambda(t) \cdot \cL_{\text{sparse}}$ to the frozen-encoder ELBO (schedule in Appendix~\ref{app:architecture}). Entropy is preferred over column-level magnitude penalties for module-partition alignment; Appendix~\ref{app:entropy_vs_grouplasso} provides the mechanistic argument, an empirical comparison, and a discussion of how this column-level penalty relates to the basis-function group-lasso analyzed in Theorem~\ref{thm:finite}. An alternative approach would directly sparsify a dense decoder's Jacobian $\partial \hat{\vx}/\partial \vz$; we discuss why the additive architecture is preferable in Appendix~\ref{app:jacobian_sparsity}.

\textbf{Inference: regime-association interpretation.}
After training, MOSAIC ranks the learned latent factors by the regime KL divergence (Definition~\ref{def:driver}); under the mean-shift regime contrasts in our experiments, Cohen's $d$ provides a stable finite-sample proxy~\citep{cohen1988statistical}. The selected $z_{j^*}$ is interpreted through its influence column $A_{:,j^*}$, whose largest entries name the observed variables that give $z_{j^*}$ its physical implication.

\textbf{Parallel transition prior.}
The Stage~1 temporal prior requires component-wise Jacobians for the change-of-variables term. Reference implementations~\citep{yao2021learning, yao2022temporally, brehmer2022weakly, song2024causal} loop over $n$ dimensions sequentially, which dominates wall-clock cost at scientific scale (synthetic training at $\hat n=8$ takes ${\sim}65$ GPU-hours per seed). We replace the sequential loop with a single batched computation that yields the same log-det-Jacobian, giving ${\geq}277\times$ wall-clock speedup on NVIDIA A40 (Table~\ref{tab:transition_bench}; numerical equivalence verified in Appendix~\ref{app:architecture}).

\section{Experiments}
\label{sec:experiments}

We organize experiments around the \textbf{two abstract capabilities} of MOSAIC: identifying latent factors and recovering their observation-level supports. 

\textbf{Setup.} We evaluate on six benchmarks across various domains with sample-to-dimension ratios spanning $13$ to $14{,}000$: synthetic ($D=30$, $N=420$K, full ground truth), RNA molecular dynamics (MD, $D=28$, $N=156$K), tokamak-inspired control ($D=12$, $N=4.6$K), OMNI solar wind~\citep{king2005omni} ($D=21$), climate ENSO~\citep{huang2017extended} ($D=47$), and TEP~\citep{downs1993plant} ($D=52$). The synthetic benchmark uses monotonic invertible nonlinear mixing with five function families (Appendix~\ref{app:synthetic}). We run 5 seeds on synthetic and RNA, 3 on cross-domain. We additionally test structural localization on HP35 protein dynamics ($D=132$; Section~\ref{sec:hp35}).

\textbf{Baselines.} We compare against three families (citations in Section~\ref{sec:related}): temporal CRL methods (TDRL, iVAE, SlowVAE, $\beta$-VAE, CtrlNS), linear sparse methods (Sparse PCA, ICA, PCA), and the regime-conditioned information bottleneck SPIB. For all methods, regime-association ranking uses a shared Cohen's $d$ pipeline; localization is read from additive decoder probes (MOSAIC), dense-decoder Jacobians (CRL methods), or linear loadings (Sparse PCA/ICA/PCA).

\textbf{Metrics.} Three metrics target different aspects of recovery: \textbf{MCC}, the mean correlation coefficient between learned and ground-truth latents under Hungarian matching, measures latent identifiability; \textbf{Z@top3} checks whether the three most regime-discriminative learned latents each correspond to a true regime-varying factor; and \textbf{$X_Z$@top3} measures support precision of those three latents against ground-truth observation supports, gated at top-3 mass $\geq 0.50$ to penalize diffuse decoders whose argmax happens to overlap by chance. Full definitions are in Appendix~\ref{app:hyperparams}.

\vspace{-8pt}
\begin{figure}[t]
\centering
\includegraphics[width=1\textwidth]{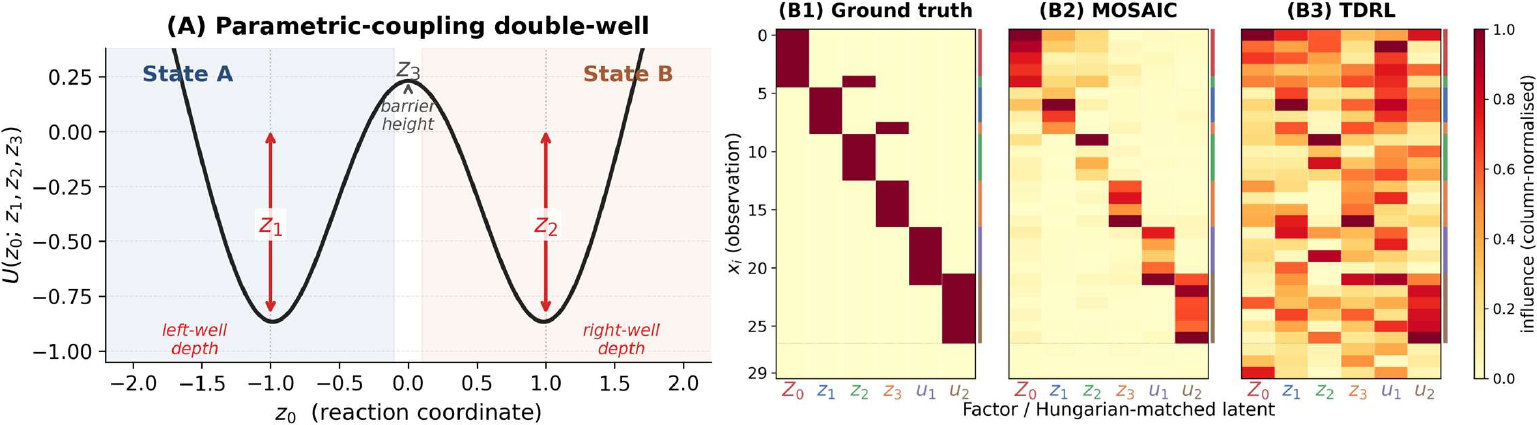}
\caption{Physics-inspired double-well benchmark. \textbf{(A)} A reaction coordinate $z_0$ moves between two metastable basins (State A / State B); $z_1, z_2, z_3$ modulate left-well depth, right-well depth, and barrier height. \textbf{(B1--B3)} Column-normalized influence matrices over six ground-truth factors ($z_0,z_1,z_2$ regime-varying; $z_3,u_1,u_2$ regime-invariant). Learned latents in B2 and B3 are Hungarian-matched to the ground truth in B1. MOSAIC (B2) recovers the diagonal module structure; TDRL (B3) does not.}
\label{fig:synthetic}
\vspace{-8pt}
\end{figure}

\subsection{Physics-Inspired Synthetic Energy-Landscape Benchmark}
\label{sec:synthetic}

We validate both capabilities on a controlled double-well benchmark with known ground-truth factors and observation-level supports. A reaction coordinate $z_0$ moves between two basins (local energy minima) while $z_1, z_2, z_3$ modulate left-well depth, right-well depth, and barrier height; State A/B occupancy defines the regime label. The benchmark generates $D=30$ observations from $n_{\text{true}}=6$ latents under monotonic nonlinear mixing, with shared-channel overlaps that violate the strict disjoint-support assumption A1$'$ to probe the more realistic partial-overlap regime of Proposition~\ref{prop:partial_overlap} (Figure~\ref{fig:synthetic}; full specification in Appendix~\ref{app:synthetic}).

MOSAIC is the only method that satisfies both abstract capabilities (Table~\ref{tab:synthetic}): it identifies the latent factors with the best MCC ($0.912$), correctly identifies the regime-associated factors in all 5 seeds, and recovers their observation-level supports with the best $X_Z$@top3 ($0.978$). The baselines fall into three families that each cover only one capability. CRL methods (TDRL, iVAE, SlowVAE, $\beta$-VAE) use dense decoders, so although they recover useful latent factors (TDRL MCC $0.850$), their supports are diffuse and gated $X_Z$@top3 is zero. Linear sparse methods (Sparse PCA, ICA, PCA) produce readable loadings (un-gated $X_Z$@top3 around $0.6$--$0.7$) but cap MCC below $0.78$ for lack of nonlinear capacity. SPIB's information-bottleneck objective does not align with driver recovery: Z@top3 and $X_Z$@top3 are uniformly zero across 36 runs (Appendix~\ref{app:spib_sweep}). After correcting two tensor-shape errors in our CtrlNS adapter, a $3\times3$ search over transition weights and learning rates yields MCC $0.791$, validation MSE $0.0389$, and native gate accuracy $0.539$ for the configuration selected by validation gate accuracy, with MSE breaking ties. CtrlNS learns useful latent structure, but its internally discovered partition need not align with the specified scientific regime. Post-hoc top-3 thresholding of TDRL's dense influence map only reaches $X_Z$@top3 $=0.444$, confirming that sparsity alone cannot recover module structure without joint training. This $X_Z$@top3 of $0.978$ is achieved on a benchmark that deliberately violates A1$'$ (Proposition~\ref{prop:partial_overlap}).

\vspace{-8pt}
\begin{table}[t]
  \caption{Synthetic benchmark with monotonic nonlinear mixing. $X_Z$@top3 uses the $0.50$ concentration gate; stochastic methods report mean $\pm$ std over 5 seeds. iVAE uses the true regime label $c_t$. The tuned CtrlNS result is reported in the text; linear baselines report un-gated $X_Z$@top3. Detailed hyperparameters are in Appendix~\ref{app:hyperparams}; SPIB diagnostics are in Appendix~\ref{app:spib_sweep}.}
  \label{tab:synthetic}
  \centering
  \small
  \begin{tabular}{lccc}
    \toprule
    Method & MCC $\uparrow$ & Z@top3 perfect seeds $\uparrow$ & $X_Z$@top3 (gate 0.50) $\uparrow$ \\
    \midrule
    \textbf{MOSAIC (ours)} & $\mathbf{0.912 \pm 0.010}$ & $\mathbf{5/5}$ & $\mathbf{0.978 \pm 0.044}$ \\
    iVAE ($u = c_t$) & $0.821 \pm 0.007$ & 3/5 & $0.556 \pm 0.157$ \\
    TDRL & $0.850 \pm 0.026$  & 4/5 & $0.000 \pm 0.000$ \\
    TDRL + post-hoc top-3 & $0.850 \pm 0.026$  & 4/5 & $0.444 \pm 0.211$ \\
    SlowVAE & $0.696 \pm 0.046$ & 2/5 & $0.000 \pm 0.000$ \\
    $\beta$-VAE ($\beta{=}4$, $\gamma{=}0$) & $0.760 \pm 0.060$ & 0/5 & $0.000 \pm 0.000$ \\
    Sparse PCA ($\alpha{=}10$) & 0.775 & --- & 0.610 (no gate) \\
    ICA & 0.765 & --- & 0.610 (no gate) \\
    PCA & 0.684 & --- & 0.670 (no gate) \\
    SPIB & $0.304 \pm 0.019$ & 0/5 & $0.000 \pm 0.000$ \\
    \bottomrule
  \end{tabular}
  \vspace{-8pt}
\end{table}

\subsection{RNA Molecular Dynamics: Interpreting a Regime-Associated Latent}
\label{sec:rna}

We now apply MOSAIC to real scientific data where recovered supports can be read as substantive scientific claims. The mechanistic question is: \textit{which learned latent factor is most associated with the C5-G10 closing-pair regime shift, and which physical residues give that factor its meaning?}

\textbf{Setup.} We perform MD simulations of the cUUCGg tetraloop (\texttt{GGCACUUCGGUGCC}) at 345K, 400K, and 500K (14 trajectories total; full simulation protocol in Appendix~\ref{app:rna}). Per-residue distance features give $D{=}28$ named observed variables, which group into Loop (U6--G9), ClosingPair (C5, G10), and Stem residues for interpretation; these group labels are not used to supervise the latents. Regimes are defined by the closing-pair native-contact fraction $Q_{\text{cp}}$ on contacts touching C5 or G10, with intact ($Q_{\text{cp}}>0.8$) and disrupted ($Q_{\text{cp}}<0.65$) frames forming a balanced dataset of $156{,}170$ lag-2 windows ($N/D \approx 5{,}571$).

\textbf{MOSAIC identifies a Loop-localized correlate of the closing-pair shift.} The Cohen's-$d$ ranking selects a single regime-discriminative latent $z_3$ (Figure~\ref{fig:rna}B). Its influence column $A_{:,3}$ concentrates on Loop residues, with the top-4 entries 100\% Loop (Figure~\ref{fig:rna}C1). This is consistent with the cUUCGg folding pathway studied by \citet{chen2013rna}, in which loop fluctuations couple to closing-pair disruption. Across 5 seeds, Loop ranks first in 4 runs and is the dominant residue group in all 5 (the fifth seed surfaces stem-loop coupled residues that remain physically consistent with loop dynamics; details in Appendix~\ref{app:rna}). Regime accuracy is $0.912 \pm 0.027$.

\textbf{Comparison to baselines.} The same three baseline families from Section~\ref{sec:synthetic} reappear here. CRL methods (TDRL, iVAE, SlowVAE, $\beta$-VAE) yield a regime-associated latent but localize it to the wrong or diffuse residue groups; Figure~\ref{fig:rna}C2 shows TDRL producing a sparse-looking column whose mass sits on Stem rather than Loop ($X_Z$@selected-latent $\leq 0.625$ across CRL methods). This failure mode is informative: sparsity alone is insufficient without the additive main-effect projection that MOSAIC uses to align sparse mass with the correct latent. Table~\ref{tab:rna_main} summarizes the comparison; full per-seed details are in Appendix~\ref{app:rna}. The Loop localization is not circular: redefining regimes by outer-stem contacts shifts the dominant residue group from Loop to Stem (Appendix~\ref{app:rna}, Table~\ref{tab:rna_cross_regime}).

\vspace{-8pt}
\begin{table}[h]
  \centering
  \small
  \caption{RNA baseline comparison ($X_Z$@selected-latent, top-8 precision against Loop).}
  \label{tab:rna_main}
  \begin{tabular}{lcc}
    \toprule
    Method family & Regime Acc & $X_Z$@sel $\uparrow$ \\
    \midrule
    \textbf{MOSAIC (5s)} & $\mathbf{0.912 \pm 0.027}$ & $\mathbf{0.880 \pm 0.055}$ \\
    CRL best (TDRL, 4s) & $0.912 \pm 0.030$ & $0.625 \pm 0.102$ \\
    Linear best (SpPCA $\alpha{=}10$) & $0.898$ & $0.750$ \\
    SPIB (5s) & $0.645 \pm 0.065$ & $0.400 \pm 0.137$ \\
    \bottomrule
  \end{tabular}
  \vspace{-8pt}
\end{table}

\subsection{HP35 Protein Dynamics: Localizing an H1 Folding Contrast}
\label{sec:hp35}

We use the HP35 Anton trajectory~\citep{piana2012protein} to test whether a specified folding contrast can be localized to its structural region (Figure~\ref{fig:hp35_schematic}).

\textbf{Setup.} Each frame is represented by 132 backbone dihedral features. Before training, H1 folding is specified as the target contrast using native contacts from an independent distance representation. The methods receive the global dihedral inputs and regime labels, but neither the distances nor H1 annotations. We select the component with the largest absolute Cohen's $d$, then evaluate its separation from the second-best component and count how many of its top-40 features lie in H1, which contains 40 of the 132 input features. Features are ranked by decoder influence for MOSAIC and by loadings for linear methods.

\textbf{MOSAIC identifies a distinct H1-associated latent.} Its regime separation is $10.06$ times that of the second-best component, and $30/40$ selected features lie in H1 (Table~\ref{tab:hp35}).

\textbf{Comparison to baselines.} PCA and FastICA also identify a leading component, but only $40\%$ and $55\%$ of their selected features lie in H1. Sparse PCA with $\alpha=1$ places more than half outside H1, while $\alpha=10$ improves H1 concentration but gives nearly tied H1 and H2 components, with a separation ratio of only $1.02$. MOSAIC thus combines distinct regime association with concentrated H1 support, localizing the structural change associated with the specified folding contrast.

\vspace{-8pt}
\begin{figure}[t]
  \centering
  \includegraphics[width=\textwidth]{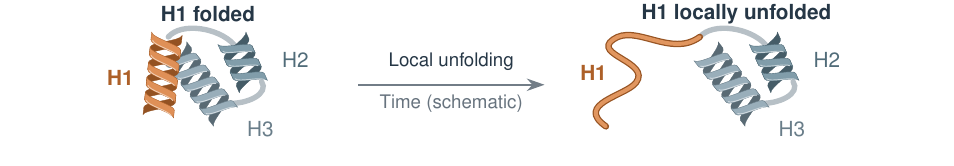}
  \caption{\textbf{HP35 protein dynamics.} Schematic time evolution from an H1-folded to an H1-locally-unfolded conformation, with H1 highlighted in orange.}
  \label{fig:hp35_schematic}
  \vspace{-8pt}
\end{figure}

\vspace{-8pt}
\begin{table}[ht]
  \centering
  \small
  \caption{HP35 H1 localization. Separation is the ratio of the largest to the second-largest absolute Cohen's $d$. H1 localization counts the H1 features among the selected component's top-40 features.}
  \label{tab:hp35}
  \begin{tabular}{lcc}
    \toprule
    Method & Separation & H1 in top-40 \\
    \midrule
    \textbf{MOSAIC} & $10.06\times$ & $30/40\ (75.0\%)$ \\
    PCA & $4.17\times$ & $16/40\ (40.0\%)$ \\
    Sparse PCA ($\alpha=1$) & $1.37\times$ & $19/40\ (47.5\%)$ \\
    Sparse PCA ($\alpha=10$) & $1.02\times$ & $27/40\ (67.5\%)$ \\
    FastICA & $11.28\times$ & $22/40\ (55.0\%)$ \\
    \bottomrule
  \end{tabular}
  \vspace{-8pt}
\end{table}

\vspace{-8pt}
\begin{table}[t]
  \caption{Cross-domain localization (MOSAIC, 3 seeds per dataset). Ratio = winner-group mean influence / runner-up.}
  \label{tab:cross_domain}
  \centering
  \small
  \resizebox{\textwidth}{!}{%
  \begin{tabular}{llccccc}
    \toprule
    Dataset & Domain & $D$ & Regime Acc & Top-1 match & Winner group & Ratio \\
    \midrule
    OMNI       & Solar wind    & 21 & $0.940 \pm 0.054$ & 3/3 & 3/3 (Geomag/Coupling) & $1.00$--$1.83$ \\
    Disruption & Tokamak (synthetic) & 12 & $0.911 \pm 0.003$ & 3/3 & 3/3 (MHD/Density) & $1.38$--$2.13$ \\
    Climate    & ENSO          & 47 & $0.888 \pm 0.027$ & 3/3 & 3/3 (Ni\~no bands) & $1.16$--$2.00$ \\
    TEP        & Chemical proc. & 52 & $0.729 \pm 0.065$ & 3/3 & 3/3 (Stripper) & $1.07$ \\
    \bottomrule
  \end{tabular}%
  }
  \vspace{-4pt}
\end{table}

\subsection{Cross-Domain Generalization}
\label{sec:cross_domain}
We now test whether the same pipeline generalizes across domains with different physics. We apply MOSAIC to three empirical datasets (OMNI solar wind~\citep{king2005omni}, ENSO climate~\citep{huang2017extended}, Tennessee Eastman Process~\citep{downs1993plant}), and a physics-motivated synthetic tokamak benchmark.

\textbf{Setup.} Without ground-truth latent supports, we evaluate two coarser criteria. \emph{Top-1 match}: does the regime-discriminative latent's largest-influence variable lie in the domain-expected group? \emph{Winner group}: does the per-group mean influence rank the expected group first? We also report the \emph{ratio} of the winner group's mean influence to the runner-up's, measuring localization sharpness. Expected groups are determined by domain knowledge (e.g., Ni\~no bands for ENSO, MHD/density for tokamak) and used only for post hoc evaluation.

\textbf{MOSAIC matches the domain-expected group in all 12 seeds across the four datasets} (Table~\ref{tab:cross_domain}). Localization sharpness varies with how cleanly the regime label maps to a single physical mechanism: the disruption benchmark gives ratios up to $2.13$, while TEP's multi-fault regime (IDV-1, 2, 4, 5) converges onto the shared downstream Stripper signature with ratio $1.07$. As a reference point, SparsePCA and FastICA each achieve 3/3 winner-group recovery on three of the four datasets, confirming that the regime signal is detectable at group level by simple methods; however, they lack identifiable latents and concentrated supports. PCA fails on OMNI and Climate by selecting variance-dominant rather than regime-relevant variables. SPIB succeeds on 5/9 runs, failing all three OMNI seeds. MOSAIC is the only method achieving 12/12 across all four datasets. Full per-method tables are in Appendix~\ref{app:crossdomain}.

\textbf{Case Study: low-$N/D$ regime.} Climate ($N/D = 13$) lies below the recovery range observed in our synthetic sweep (Appendix~\ref{app:synthetic:nd}). This sweep is a sample-size stress test for a fixed generator, not a transferable reliability threshold: recovery also depends on temporal dependence, signal strength, sparsity, interactions, and collinearity. The top regime-associated latent concentrates on equatorial Pacific SST cells (Figure~\ref{fig:climate}), recovering a spatial region broader than the single grid cell defining the regime label, consistent with the underlying ENSO mode. Climate thus demonstrates regime prediction and coarse group localization, with sensor-level support ground truth unavailable. For new datasets without support labels, we propose checking effective sample size and regime transitions, the dense--additive residual gap on held-out temporal blocks, support concentration against a block-permutation null, and support stability across blocked refits against a random-support null. These diagnostics screen for failure modes without certifying exact recovery.

\textbf{Takeaway.} Each experiment stresses a different aspect of the pipeline: the synthetic benchmark probes both capabilities under controlled nonlinear mixing with deliberate support overlap; RNA tests whether the recovered support constitutes a substantive scientific claim on a real molecular system, where MOSAIC's advantage is largest (Tables~\ref{tab:synthetic},~\ref{tab:rna_main}); and the four cross-domain datasets are non-redundant (Climate: low $N/D$; TEP: multi-fault aggregation; Disruption: signal-strength separation; OMNI: geomagnetic regimes). MOSAIC is the only method that handles all six benchmarks with a single pipeline and shared hyperparameters. HP35 further tests localization of a specified folding contrast to a protein region.

\vspace{-8pt}
\begin{figure}[t]
\centering
\includegraphics[width=\textwidth]{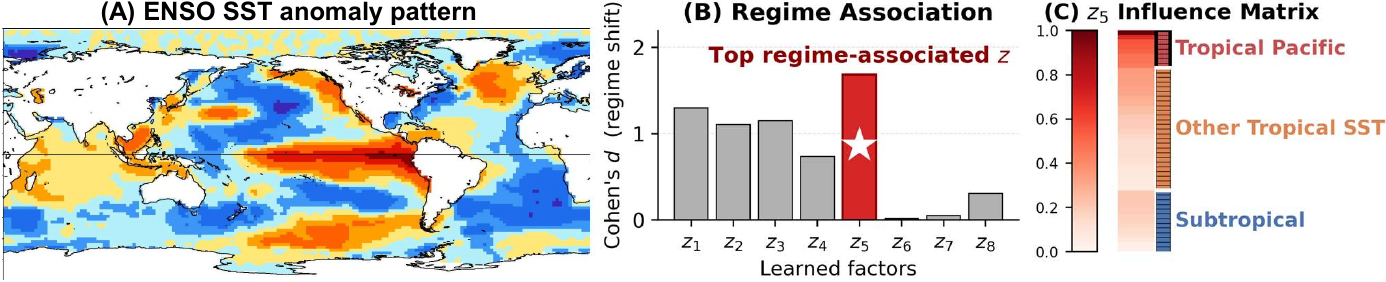}
\caption{\textbf{Climate (ENSO) cross-domain localization} on NOAA ERSSTv5 sea surface temperature data~\citep{huang2017extended,huang2017ersstv5data}. \textbf{(A)} The canonical ENSO SST anomaly pattern, with the warm equatorial Pacific band marking the physical mode underlying the regime label. \textbf{(B)} Cohen's $d$ across 8 learned latents identifies $z_5$ as the top regime-associated factor. \textbf{(C)} Influence of $z_5$ over the 47 SST grid cells, ordered by group: support concentrates on the Tropical Pacific group corresponding to the equatorial Pacific band in (A), extending beyond the single grid cell that defines the regime label.}
\label{fig:climate}
\vspace{-8pt}
\end{figure}

\section{Conclusion and Limitations}
\label{sec:conclusion}

MOSAIC closes a gap between temporal CRL, which identifies latents but cannot determine what they represent, and scientific representation learning, which operates on named variables but lacks identifiability. By combining both, MOSAIC turns anonymous identifiable latents into module-level scientific interpretations, validated across diverse domains. 

\textbf{Ablation and robustness.} Ablation on the synthetic benchmark (Appendix~\ref{app:ablation}) shows a clean decomposition: MCC depends only on the temporal prior, while $X_Z$ requires all three components together; removing any single component breaks the alignment between sparse supports and correctly identified latents. Additional robustness experiments confirm graceful degradation under varying latent dimensionality (Appendix~\ref{app:synthetic:zdim}), regime-label noise up to $30\%$ (Appendix~\ref{app:synthetic:regime_noise}), unsupervised regime discovery via $k$-means (Appendix~\ref{app:synthetic:auto_regime}), and alternative regime definitions (Appendix~\ref{app:synthetic:regime_def}).

\textbf{Scope and limitations.} MOSAIC interprets a binary or one-vs-reference regime contrast at a time; multi-regime joint analysis is left to future work (Appendix~\ref{app:synthetic:multi_regime}). Entropy regularization is used instead of the group-lasso variant of Theorem~\ref{thm:finite} because it directly enforces within-column concentration, at the cost of an open finite-sample theorem (Appendix~\ref{app:entropy_vs_grouplasso}). Regime association is distributional, not interventional, restricted to mean shifts under Cohen's $d$; channels with weak mixing coefficients or purely higher-order dependence may fall outside the recovered support. MOSAIC complements rather than replaces existing tools: SPIB can supply regime labels, and pairwise causal-discovery methods can operate within recovered modules. Nonlinear optimization introduces seed-to-seed variance (reported as standard deviations throughout). The goal is hypothesis generation, not autonomous causal decisions, with potential applications in space weather forecasting, fusion diagnostics, climate monitoring, and chemical-process safety.

\clearpage
\section*{Acknowledgments}
I thank my parents and friends for being there for me during a particularly difficult period in my life, when this work was carried out. Their companionship, care, and encouragement helped me through that time and gave me the strength to complete this work.

This work is supported by the National Science Foundation (NSF) Grant \#2312862, NSF-Simons SkAI Institute, NSF CAREER \#2440542, NSF \#2533996, National Institutes of Health (NIH) \#R01AG091762, NSF ACCESS Computing Resources, NAIRR, NRP, a Google Research Scholar Award, and Cisco gift grant.

\section*{Competing Interests}
The authors declare no competing interests.

\bibliographystyle{plainnat}
\bibliography{references}


\appendix

\section{Architecture Details}
\label{app:architecture}

\textbf{Encoder.} The encoder is a 4-layer MLP with LeakyReLU(0.2) activations and no normalization layers. The hidden dimension is domain-dependent (64 to 128); the final layer maps to $[\boldsymbol{\mu}, \log \boldsymbol{\sigma}^2] \in \real^{2n}$, from which we reparametrize $\vz = \boldsymbol{\mu} + \boldsymbol{\sigma} \odot \boldsymbol{\epsilon}$ with $\boldsymbol{\epsilon} \sim \mathcal{N}(\mathbf{0}, \mathbf{I})$.

\textbf{Stage~1 decoder (dense).} A standard MLP $g_\theta: \real^n \to \real^D$ used during Stage~1 to learn identifiable latent factors. This decoder is discarded after Stage~1.

\textbf{Stage~2 decoder (additive).} Each of the $n$ component functions $f_j$ is a 2-layer MLP: $z_j \in \real \to \text{Linear}(1, H) \to \text{LeakyReLU} \to \text{Linear}(H, D)$, where $H$ is the per-component hidden dimension (hidden\_per\_z). The total decoder output is $\hat{\vx} = \sum_{j=1}^{n} f_j(z_j) + \mathbf{b}$, with a learnable bias $\mathbf{b} \in \real^{D}$. The encoder and transition prior are frozen during Stage~2; only the additive decoder parameters are optimized, using \texttt{AdamW} restricted to parameters with \texttt{requires\_grad=True}.

\textbf{Per-dimension residual prior.} For dimension $j$, the Stage~1 prior assumes a Laplace density on the residual $r_{t,j} = h_j(\vz_{t-L:t-1}, z_{t,j}, \mathbf{e}_{c_t})$, where $h_j$ is a per-dimension transition MLP that receives the past $L$ lag vectors, the current factor $z_{t,j}$ (but not $z_{t,k}$ for $k \neq j$), and the regime embedding. The prior density is
\begin{equation}
    \log p_j(z_{t,j} \mid \vz_{t-L:t-1}, c_t) = -\frac{|r_{t,j}|}{b_j} - \log(2 b_j) + \log\left|\tfrac{\partial r_{t,j}}{\partial z_{t,j}}\right|,
    \label{eq:prior_components}
\end{equation}
where $b_j$ is a learnable scale. Because $h_j$ depends on $z_{t,j}$ through nonlinear layers, the Jacobian $\partial r_{t,j}/\partial z_{t,j}$ is not identically 1 and contributes a non-trivial log-determinant term, which is what motivates the batched autograd computation described below.

\textbf{Training objectives.} The Stage~1 loss combines reconstruction, a standard-Gaussian KL, and the temporal-prior KL,
\begin{equation}
    \cL_{\text{Stage1}} = \cL_{\text{recon}} + \beta \cdot \KL\bigl(q_\phi(\vz_t \mid \vx_t) \,\|\, \mathcal{N}(\mathbf{0}, \mathbf{I})\bigr) + \gamma \cdot \KL\bigl(q_\phi(\vz_t \mid \vx_t) \,\|\, p(\vz_t \mid \vz_{t-L:t-1}, c_t)\bigr),
    \label{eq:stage1_loss}
\end{equation}
with $\beta$ weighting the standard-Gaussian KL and $\gamma$ weighting the temporal-prior KL, evaluated for $t \geq L$. The two KL terms serve distinct roles: the standard-Gaussian term ($\beta = 2 \times 10^{-3}$) acts as an empirical regularizer that bounds latent magnitudes and prevents posterior collapse, while the temporal-prior term ($\gamma = 2 \times 10^{-2}$) supplies the regime-conditioned nonstationarity signal required for identifiability. This objective is not a single-prior ELBO; it follows the dual-regularization design used in prior temporal CRL implementations~\citep{yao2022temporally, yao2021learning}.

The Stage~2 loss is $\cL_{\text{Stage2}} = \cL_{\text{recon}} + \lambda(t) \cdot \cL_{\text{sparse}}$. Because the encoder and prior are frozen, the temporal-prior KL is constant with respect to the additive decoder parameters and contributes no gradients; only reconstruction and the sparsity penalty shape the decoder. $\lambda(t)$ follows a three-stage schedule: zero during a 5-epoch warmup, linear ramp to $\lambda_{\max} = 50$ over the next 20 epochs, and constant $\lambda_{\max}$ thereafter. The active-factor mask $\mathcal{J}_{\text{alive}} = \{j : \sum_i A_{ij} > 0.01 \cdot \max_k \sum_i A_{ik}\}$ excludes near-dead factors from the entropy penalty in Eq.~\eqref{eq:entropy_sparsity}.

\textbf{Transition prior.} The parallel MLP takes the lagged latent history $\vz_{t-L:t-1}$ together with the regime embedding $\mathbf{e}_{c_t}$ as input and outputs the per-dimension transition shift. For each target dimension $j$, the corresponding input width is $Ln + 1 + H_{\text{emb}}$, where $H_{\text{emb}}$ is the regime-embedding hidden dimension, and all $n$ dimensions are processed simultaneously via batched \texttt{einsum}: $\mathbf{h} = \sigma(\text{einsum}(\texttt{`bnf,nfh->bnh'}, \text{input}, \mathbf{W}_1) + \mathbf{b}_1)$, then $\mathbf{r} = \text{einsum}(\texttt{`bnh,nh->bn'}, \mathbf{h}, \mathbf{W}_2) + \mathbf{b}_2$.

The log-det-Jacobian required by the change-of-variables formula reduces to the diagonal $\partial r_t^{(j)} / \partial z_t^{(j)}$ for each $j$. Because $z_t^{(j)}$ enters only into $r_t^{(j)}$ in the additive prior, the scalar derivative equals the gradient of $\sum_{b,j} r_{b,j}$ with respect to the corresponding input coordinate, which we obtain in a single \texttt{torch.autograd.grad} call. This avoids instantiating the full $(BW, 1, BW, F)$-shaped Jacobian that TDRL's reference implementation builds per dimension.

\textbf{Asymptotic complexity.}
Let $BW = B \cdot (T{-}L)$ denote the total number of time steps per batch, and $F = H_{\text{emb}} + Ln + 1$ the input width per transition MLP, where $H_{\text{emb}}$ is the regime-embedding hidden dimension. Table~\ref{tab:transition_complexity} compares the two implementations along four axes.

\begin{table}[h]
  \caption{Asymptotic complexity of the transition prior. v1 = TDRL's reference implementation (sequential loop + \texttt{autograd.functional.jacobian}); v2 = ours (parallel MLP + scalar-sum autograd). The dominant savings are the elimination of the $BW$ factor in the Jacobian cost and the elimination of $O(n)$ sequential GPU kernel launches. Here $n$ denotes the latent dimension; $D$ refers exclusively to the observation dimension throughout. The forward FLOPs include the input-to-hidden ($F \to H$) and hidden-to-hidden ($H \to H$) layers of the 3-layer transition MLP.}
  \label{tab:transition_complexity}
  \centering
  \small
  \resizebox{\textwidth}{!}{%
  \begin{tabular}{lccl}
    \toprule
    Quantity & v1 (TDRL) & v2 (ours) & Ratio \\
    \midrule
    Forward FLOPs    & $n \cdot BW \cdot (FH + H^2)$ (sequential) & $n \cdot BW \cdot (FH + H^2)$ (batched)   & $1\times$ same work \\
    Jacobian FLOPs   & $n \cdot BW^2 \cdot F \cdot H$      & $n \cdot BW \cdot F \cdot H$       & $BW\times$ savings \\
    GPU kernel calls & $O(n \cdot L)$ sequential           & $O(L)$ batched                     & $n\times$ launch reduction \\
    Peak memory      & $O(n \cdot BW^2 \cdot F)$           & $n \cdot BW \cdot F$            & $BW\times$ savings \\
    \bottomrule
  \end{tabular}%
  }
\end{table}

Two mechanisms account for the speedup. First, \texttt{autograd.functional.jacobian} constructs the full $(BW, 1, BW, F)$-shaped Jacobian even though only the $BW$ diagonal entries are needed, inflating both FLOPs and memory by a factor of $BW$; the scalar-sum trick extracts the diagonal directly in $O(BW)$ work. Second, the Python \texttt{for} loop over $n$ latent dimensions serializes $n$ small kernel launches, each under-utilizing the GPU; \texttt{einsum} fuses them into a single launch. The asymptotic speedup is therefore ${\sim}n \cdot \min(BW, \text{GPU-utilization gap})$, predicting that the ratio grows with both $n$ and batch size, a prediction verified in the benchmark below.

\begin{table}[h]
  \caption{Transition prior wall-clock benchmark on NVIDIA A40 (30-iteration mean, 5 warmup iterations). Same symbols as Table~\ref{tab:transition_complexity}; $Z$ denotes the latent dimension ($n$ in the complexity notation).}
  \label{tab:transition_bench}
  \centering
  \small
  \resizebox{\textwidth}{!}{%
  \begin{tabular}{lccccrrrrrr}
    \toprule
    Configuration & $Z$ & $B$ & $T$ & $H$ & v1 time & v2 time & Speedup & v1 mem & v2 mem & Mem ratio \\
    \midrule
    Synthetic     &  8 & 256 & 3 & 128 & 1635\,ms & 2.9\,ms & $563\times$  & 909\,MB  & 38\,MB  & $24\times$ \\
    RNA           &  4 & 256 & 3 & 128 &  814\,ms & 2.9\,ms & $277\times$  & 511\,MB  & 27\,MB  & $19\times$ \\
    Cross-domain  &  8 & 256 & 3 & 128 & 1637\,ms & 2.9\,ms & $560\times$  & 909\,MB  & 38\,MB  & $24\times$ \\
    Stress test   & 64 & 256 & 3 &  64 & 7327\,ms & 3.0\,ms & $2445\times$ & 3303\,MB & 122\,MB & $27\times$ \\
    \bottomrule
  \end{tabular}%
  }
\end{table}

The empirical scaling matches the complexity analysis: v1 runtime grows linearly with $Z$ ($814 \to 1637 \to 7327$\,ms as $Z$ goes $4 \to 8 \to 64$, consistent with the $D\times$ kernel-launch overhead), while v2 runtime is essentially constant (${\sim}3$\,ms) because the einsum does not yet saturate the GPU at these sizes. The memory ratio of ${\sim}24\times$ across configurations matches the predicted $BW = 256 \times (T{-}L) \approx 256$ savings, modulated by the width factor $F$. Both optimizations are algorithmic (not approximations): v2 computes the exact same log-det-Jacobian as v1, verified numerically to $2.6 \cdot 10^{-4}$ max absolute difference. End-to-end, synthetic training at the speedup-benchmark setting ($Z{=}8$, 80 epochs, ${\sim}128\text{K}$ steps) drops from ${\sim}65$ hours to ${\sim}90$ minutes per seed. The main synthetic experiments use $\hat n = 12$ rather than the $Z=8$ benchmark setting, so their wall-clock is moderately larger: the main $\hat n=12$ synthetic setting takes on the order of 1--2 hours per seed, which is what makes the 5-seed cross-ablation experimental design feasible on a single A40.

\section{Hyperparameters}
\label{app:hyperparams}

\begin{table}[h]
  \caption{Hyperparameter settings. All experiments share $\beta = 2{\times}10^{-3}$ (past-KL weight), $\gamma = 2{\times}10^{-2}$ (future-KL weight), lr $= 5{\times}10^{-4}$, weight decay $= 10^{-4}$, AdamW optimizer. Stage~1 uses a dense decoder; Stage~2 freezes the encoder and uses an additive decoder with the entropy-based sparsity penalty of Eq.~\eqref{eq:entropy_sparsity}.}
  \label{tab:hyperparameters}
  \centering
  \small
  \resizebox{\textwidth}{!}{%
  \begin{tabular}{lcccccc}
    \toprule
    Hyperparameter & Synthetic & RNA & OMNI & Disruption & Climate & TEP \\
    \midrule
    Input dim $D$ & 30 & 28 & 21 & 12 & 47 & 52 \\
    Learned latent dim $\hat n$ ($z_{\dim}$) & 12 & 4 & 5 & 5 & 8 & 8 \\
    Encoder hidden & 128 & 128 & 64 & 64 & 128 & 128 \\
    Stage~2 $H$ (hidden/z) & 4 & 4 & 4 & 4 & 4 & 4 \\
    $\lambda_{\max}$ (sparsity) & 50 & 50 & 50 & 50 & 50 & 50 \\
    Stage~1 epochs & 80 & 80 & 80 & 80 & 80 & 80 \\
    Stage~2 epochs & 80 & 80 & 80 & 80 & 80 & 80 \\
    Batch size & 256 & 256 & 256 & 256 & 256 & 256 \\
    \bottomrule
  \end{tabular}%
  }
\end{table}

\subsection{Experimental Protocol Details}
\label{app:setup_details}

Baselines include TDRL~\citep{yao2022temporally}, CtrlNS~\citep{song2024causal}, $\beta$-VAE ($\beta=4,\gamma=0$), Sparse PCA ($\alpha=10$), ICA, and PCA. Regime-association ranking uses a shared Cohen's $d$ pipeline across all methods. For methods without additive decoders, localization is read from native structure: linear loadings for Sparse PCA, ICA, and PCA, and dense-decoder Jacobians for TDRL and $\beta$-VAE. For the latter, the per-sample Jacobian $\partial \hat{\vx}_t / \partial \hat{\vz}_t$ varies across the dataset; we aggregate by taking the mean of element-wise absolute values over all samples, $A_{ij} = \frac{1}{N}\sum_t |(\partial \hat{\vx}_t / \partial \hat{\vz}_t)_{ij}|$, to form the influence matrix. For SPIB, which lacks a generative decoder, we extract variable influence via a per-observation input-feature Lasso: for each learned latent $\hat{z}_k$, we fit $\hat{z}_k = \sum_i w_{ik} x_i + b_k$ with $\ell_1$ regularization and set $A_{ik} = |w_{ik}|$.

\paragraph{Metric definitions.}
Let $\mathcal{F} = \{f_1,\dots,f_M\}$ be the true latent factors ($M = 6$ on the synthetic benchmark: $z_0, z_1, z_2, z_3, u_1, u_2$), partitioned into regime-varying factors $\mathcal{D} \subseteq \mathcal{F}$ (those whose distribution shifts across regimes; here $\mathcal{D} = \{z_0, z_1, z_2\}$) and regime-invariant factors. Each true factor $f_j$ has a ground-truth observation support $\mathcal{S}_j = \mathcal{P}_j \cup \mathcal{H}_j \subseteq [D]$, where $\mathcal{P}_j$ is the primary support (observations generated only from $f_j$) and $\mathcal{H}_j$ is the shared support (observations generated from $f_j$ and at least one other factor). Let $\hat{Z} = \{\hat{z}_1,\dots,\hat{z}_{\hat n}\}$ be the learned latent factors ($\hat n \ge M$), and let
\begin{equation}
  d(\hat{z}_k) \;=\; \frac{\big|\,\mathbb{E}[\hat{z}_k \mid c{=}0] - \mathbb{E}[\hat{z}_k \mid c{=}1]\,\big|}{\sqrt{\tfrac{1}{2}\big(\mathrm{Var}[\hat{z}_k \mid c{=}0] + \mathrm{Var}[\hat{z}_k \mid c{=}1]\big)}}
\end{equation}
denote the Cohen's $d$ of $\hat{z}_k$ across the binary regime label $c$. We also form an influence matrix $A \in \real_{\ge 0}^{D \times \hat n}$ where $A_{ij}$ quantifies the contribution of $\hat{z}_j$ to observation $x_i$ (method-specific: decoder probe for MOSAIC, Jacobian for TDRL/$\beta$-VAE, or $|w_{ij}|$ from a per-observation Lasso for linear baselines).

\emph{Matching via Hungarian assignment \citep{kuhn1955hungarian}.} Given samples of $(\hat{\vz}, \vz_{\text{true}})$, we build the affinity matrix $C_{jk} = |\rho(f_j, \hat{z}_k)|$ using Pearson correlation and solve $\max_{\pi}\sum_j C_{j,\pi(j)}$ via the Hungarian algorithm. We then threshold the assignment: $\hat{z}_k$ is declared \emph{matched} to $f_j$ iff $\pi(j) = k$ \emph{and} $C_{jk} \ge \tau$ (we use $\tau = 0.5$); otherwise $\hat{z}_k$ is \textsc{unmatched}. Without $\tau$, the Hungarian solution would always return a full permutation, so low-quality correlations near zero would falsely count as recovered factors.

\emph{MCC.} The standard mean correlation coefficient over the thresholded Hungarian matching, averaged over matched pairs.

\emph{Z@top3 (regime-association recovery).} Let $\mathcal{T}_3$ be the indices of the three learned latents with the largest Cohen's $d$. Then
\begin{equation}
  \mathrm{Z@top3} \;=\; \tfrac{1}{3}\,\big|\{\hat{z}_k \in \mathcal{T}_3 : \hat{z}_k \text{ is matched to some } f_j \in \mathcal{D}\}\big|.
\end{equation}
We report Z@top3 as a fraction $a/3$, and call a run \emph{perfect} if $\mathrm{Z@top3} = 3/3$. This asks whether the three most regime-discriminative learned factors each correspond to a true regime-associated factor.

\emph{$X_Z$@top3 (support recovery for regime-associated factors).} For each $\hat{z}_k \in \mathcal{T}_3$, let $m(k) \in \mathcal{F} \cup \{\varnothing\}$ denote its Hungarian match. Define the top-3 mass ratio $\mu(k) = \sum_{i \in \mathrm{Top3}(A_{\cdot, k})} |A_{i,k}| \big/ \sum_{i=1}^{D} |A_{i,k}|$, measuring the fraction of the influence column concentrated in its top-3 entries. The per-latent precision is
\begin{equation}
  p(k) \;=\; \begin{cases}
    \dfrac{\big|\,\mathrm{Top3}(A_{\cdot, k}) \cap \mathcal{S}_{m(k)}\,\big|}{3}, & m(k) \in \mathcal{D} \text{ and } \mu(k) \ge 0.50,\\[1.2ex]
    0, & \text{otherwise,}
  \end{cases}
\end{equation}
The ``@top3'' suffix in the metric name refers to this fixed Top-3 selection. A hit on a shared observation in $\mathcal{S}_{m(k)}$ counts as a correct recovery regardless of which of the overlapping factors $m(k)$ is. Then
\begin{equation}
  X_{Z}\text{@top3} \;=\; \tfrac{1}{3}\sum_{\hat{z}_k \in \mathcal{T}_3} p(k).
\end{equation}
The concentration gate $\mu(k) \ge 0.50$ reflects the module-discovery objective: a regime-discriminative latent that spreads its influence uniformly across all observed variables, even when its top-3 argmax happens to overlap the true support, has not identified a module. The threshold $0.50$ is permissive --- MOSAIC satisfies it with mean top-3 mass $0.69 \pm 0.03$ on this benchmark (Table~\ref{tab:ablation}) --- and penalizes only decoders substantially more dispersed than MOSAIC's. Appendix~\ref{app:concentration_sensitivity} reports sensitivity to threshold choice. Unmatched or regime-invariant latents contribute 0.

\emph{A note on the Hungarian threshold.} The Hungarian threshold $\tau = 0.5$ is necessary because, without it, runs whose top-3 $\hat{z}$ by Cohen's $d$ happen to correlate weakly-but-better-than-random with regime-associated factors would look artificially good.

Regime accuracy is a logistic probe on $\hat{\vz}$. For $X_Z$@top3, when ground-truth support is not available, we recover per-latent support by the largest consecutive ratio gap in sorted $A_{ij}$ values (threshold 1.5), with fallback threshold $0.1\cdot\max_i A_{ij}$ when no gap is detected.

\subsection{Cohen's \texorpdfstring{$d$}{d} versus KL ranking}
\label{app:cohen_kl}

The theory defines the regime-discriminative coordinate using KL divergence because KL compares the full conditional distributions and is invariant under componentwise invertible reparametrizations. In experiments, we use Cohen's $d$ as a stable finite-sample proxy for two-regime mean-shift contrasts. This proxy is not part of the identifiability proof.

As a diagnostic, we compared Cohen's $d$ against a one-dimensional histogram estimator of forward KL on the evaluated Synthetic, RNA, OMNI, and TEP runs. The KL estimator uses a 64-bin shared-support histogram with $q_{0.005}/q_{0.995}$ trimming and Laplace smoothing, and we report forward KL $p_0 \| p_1$ as the primary ranking. Exact top-1 agreement is not guaranteed, but the top-set and module-level decisions are stable: the mean top-3 overlap is $86.1\%$, exact top-1 agreement holds in $8$ of $12$ evaluated runs, and the selected-module question agrees in $9$ of $12$ runs (Table~\ref{tab:cohen_kl}). Climate was omitted because its evaluation pipeline differs from the other datasets and was not compatible with the shared diagnostic. We therefore use Cohen's $d$ only as a finite-sample ranking proxy for the experimental protocol, not as a replacement for the KL-based identifiability statement.

A broader limitation applies to both Cohen's $d$ and the marginal KL criterion of Definition~\ref{def:driver}: they target marginal distributional shifts, which is sufficient but not necessary for regime-varying dynamics. A factor whose transition dynamics change across regimes while its stationary marginal remains invariant would not be detected by either criterion. In the scientific settings we evaluate, regime contrasts (folded vs.\ unfolded, storm vs.\ quiet, faulty vs.\ normal) produce clear marginal mean shifts, so this gap does not arise empirically. For regime changes that manifest primarily as variance or higher-order moment shifts, Cohen's $d$ would need to be replaced by a nonparametric two-sample statistic such as an MMD or energy distance; all six datasets in our evaluation are mean-shift dominated.

\begin{table}[h]
\centering
\small
\caption{Diagnostic comparison of Cohen's $d$ and one-dimensional histogram forward-KL ranking.}
\label{tab:cohen_kl}
\begin{tabular}{lccc}
\toprule
Scope & Top-1 agreement & Mean top-3 overlap & Module agreement \\
\midrule
12 evaluated runs & 8/12 & 86.1\% & 9/12 \\
\bottomrule
\end{tabular}
\end{table}

\section{Synthetic Benchmark Details}
\label{app:synthetic}

\textbf{Generative process.}
The synthetic benchmark generates $D = 30$ observations from $n_{\text{true}} = 6$ latent factors via monotonic invertible nonlinear mixing. Each observation channel $i \in [D]$ is assigned a primary parent $j \in [n_{\text{true}}]$ and a mixing function $g_{ij}$ drawn (cyclically across $i$) from
\begin{equation}
    \mathcal{F} = \bigl\{\, z \mapsto \tanh(5z),\; z \mapsto z^3/4,\; z \mapsto \mathrm{erf}(3z),\; z \mapsto \arctan(5z),\; z \mapsto \mathrm{sign}(z)\,|z|^{1/3} \bigr\}.
\end{equation}
The observation is $x_i = \sum_{j \in S_i} g_{ij}(z_j) + \epsilon_i$ with $\epsilon_i \sim \mathcal{N}(0, \sigma_\epsilon^2)$. The five functions span saturating ($\tanh$, $\arctan$), expanding ($z^3/4$, $\mathrm{sign}\cdot|z|^{1/3}$), and symmetric sigmoidal ($\mathrm{erf}$) shapes, defeating any single linear projection that simultaneously fits all channels generated by one factor.

\textbf{Linear-correlation calibration.}
The generator is calibrated so that each observation remains strongly associated with its assigned parent latent while limiting spurious cross-factor linear correlations; exact calibration statistics depend on the synthetic variant and are not used in the evaluation metrics. For the main synthetic benchmark used in Table~\ref{tab:synthetic}, the generator produces a near-linear parent association (mean parent $|r|\approx 0.99$) before nonlinear monotonic transformations, while support recovery is evaluated against the known parent/support map rather than against correlation values. After applying the nonlinear transforms in $\mathcal{F}$, the channel-parent linear correlation is substantially reduced, consistent with the MCC ceiling of $0.68$--$0.78$ achieved by linear baselines (Table~\ref{tab:synthetic}).

\textbf{Latent dynamics.}
The latent factors evolve under a fixed potential with regime-dependent coupling (full specification in the release code). Regime labels are assigned by the well occupancy of the reaction-coordinate factor, yielding 420\,K balanced samples (210\,K per regime).

\textbf{Ground-truth support structure.}
The 30 observed channels are partitioned into 6 ground-truth factors with 3 additional pure-noise channels. Table~\ref{tab:synthetic_support} lists the primary and shared assignments used to score module recovery. Critically, the benchmark includes shared-channel overlaps (channels 4, 8, 21) by design: the strict disjoint-support assumption A1$'$ used in Theorem~\ref{thm:module} and Theorem~\ref{thm:finite} is unrealistic for real scientific data, where physical observables routinely couple to multiple latent mechanisms. The shared channels deliberately violate A1$'$, placing the benchmark in the relaxed partial-overlap regime characterized by Proposition~\ref{prop:partial_overlap}. The MCC, $X_Z$@top3, and module-recovery results reported in Section~\ref{sec:synthetic} should therefore be read as evidence that MOSAIC operates beyond the strict A1$'$ setting and recovers structure under the more realistic partial-overlap conditions that Proposition~\ref{prop:partial_overlap} covers.

\begin{table}[h]
  \caption{Ground-truth support structure for the synthetic benchmark. The combined support $\mathcal{S}_{m(k)} = \mathcal{P}_{m(k)} \cup \mathcal{H}_{m(k)}$ (primary plus shared channels) is used to score Top-3 recovery in $X_Z$@top3; shared channels count as hits for either listed factor. Channels 27, 28, and 29 are pure noise.}
  \label{tab:synthetic_support}
  \centering
  \small
  \begin{tabular}{llll}
    \toprule
    Factor & Role & Primary channels & Shared channels \\
    \midrule
    $z_0$ & Regime-varying & $\{0,1,2,3\}$       & $\{4\}$ with $z_2$ \\
    $z_1$ & Regime-varying & $\{5,6,7\}$         & $\{8\}$ with $z_3$ \\
    $z_2$ & Regime-varying & $\{9,10,11,12\}$    & $\{4\}$ with $z_0$ \\
    $z_3$ & Invariant & $\{13,14,15,16\}$   & $\{8\}$ with $z_1$ \\
    $u_1$ & Invariant & $\{17,18,19,20\}$   & $\{21\}$ with $u_2$ \\
    $u_2$ & Invariant & $\{22,23,24,25,26\}$ & $\{21\}$ with $u_1$ \\
    \bottomrule
  \end{tabular}
\end{table}

\textbf{Module recovery thresholding.}
For each learned latent $z_j$, we sort $A_{ij}$ in descending order and cut at the largest consecutive ratio gap exceeding $1.5$; if no such gap exists in the top 15 entries, we fall back to the relative threshold $0.1 \cdot \max_i A_{ij}$. This procedure is parameter-free per column and adapts to different sparsity levels across factors.

\subsection{Interaction-strength sweep}
\label{app:synthetic:interactions}

We sweep the interaction strength $\alpha$ from 0 to 2.0, adding pairwise cross-factor interaction terms $\alpha \cdot \tanh(z_a \cdot z_b)$ while preserving the same main effects and regime labels (5 seeds per $\alpha$, z\_dim=12). This probes Theorem~\ref{thm:finite}'s effective-noise picture: larger interaction residuals should make support recovery more sample-demanding.

\begin{table}[h]
  \caption{Interaction strength sweep on the synthetic benchmark (5 seeds per $\alpha$, z\_dim=12). As $\alpha$ increases from 0 (additive, paper assumption) to 2 (strong non-additive coupling), support recovery $X_Z$ degrades monotonically from $0.978$ to $0.538$, while MCC remains comparatively stable. The $\alpha = 0$ row is the additive-mixing limit, trained and evaluated on a separate $\alpha$-sweep dataset using the same protocol as Table~\ref{tab:synthetic}. The MCC gap (0.934 vs Table~\ref{tab:synthetic}'s 0.912) reflects the two independent training trajectories; the per-seed $X_Z$ pattern (four perfect seeds plus one at $8/9 = 0.889$) coincides between the two runs under the discrete fixed-Top-3 metric, hence the matching aggregate $0.978 \pm 0.044$.}
  \label{tab:synthetic_interactions}
  \centering
  \small
  \begin{tabular}{cccc}
    \toprule
    $\alpha$ & MCC $\uparrow$ & Z@top3 perfect (out of 5) $\uparrow$ & $X_Z$@top3 (gate 0.50) $\uparrow$ \\
    \midrule
    0.0 & $0.934 \pm 0.006$ & 5/5 & $0.978 \pm 0.044$ \\
    0.1 & $0.936 \pm 0.005$ & 5/5 & $0.911 \pm 0.130$ \\
    0.3 & $0.920 \pm 0.021$ & 4/5 & $0.744 \pm 0.276$ \\
    0.5 & $0.941 \pm 0.011$ & 4/5 & $0.713 \pm 0.216$ \\
    0.7 & $0.945 \pm 0.014$ & 4/5 & $0.683 \pm 0.178$ \\
    1.0 & $0.894 \pm 0.045$ & 3/5 & $0.651 \pm 0.286$ \\
    1.5 & $0.907 \pm 0.037$ & 3/5 & $0.544 \pm 0.070$ \\
    2.0 & $0.889 \pm 0.041$ & 3/5 & $0.538 \pm 0.166$ \\
    \bottomrule
  \end{tabular}
\end{table}

The sweep is consistent with Theorem~\ref{thm:finite}'s effective-noise picture: as $\alpha$ increases, the non-additive variance contribution grows and the additive decoder's support recovery becomes harder, while the encoder's latent recovery (MCC) remains robust in a wider range. Support recovery leaves the high-reliability range between $\alpha=0.1$ and $\alpha=0.3$, then degrades more gradually, falling below $0.6$ only at the strongest interaction levels ($\alpha \geq 1.5$).

\subsection{Sample-complexity sweep (N/D)}
\label{app:synthetic:nd}

We subsample the synthetic benchmark by retaining every $k$-th frame (preserving temporal ordering) to vary $N/D$ from 5 to 5000, training MOSAIC with the same hyperparameters across all levels (3 seeds per level). The main benchmark uses the full $N=420$K samples; the subsampled levels below use up to $N=150$K.

\begin{table}[h]
  \caption{$N/D$ sweep on the synthetic benchmark (3 seeds per level). MCC and $X_Z$@top3 are seed means $\pm$ std.}
  \label{tab:synthetic_nd}
  \centering
  \small
  \begin{tabular}{rrccc}
    \toprule
    N/D & N & MCC $\uparrow$ & Regime Acc & $X_Z$@top3 $\uparrow$ \\
    \midrule
    5 & 150 & $0.558 \pm 0.051$ & $0.938 \pm 0.050$ & $0.022 \pm 0.039$ \\
    10 & 300 & $0.568 \pm 0.079$ & $0.956 \pm 0.017$ & $0.044 \pm 0.077$ \\
    25 & 750 & $0.561 \pm 0.066$ & $0.998 \pm 0.003$ & $0.044 \pm 0.077$ \\
    50 & 1500 & $0.663 \pm 0.038$ & $0.999 \pm 0.001$ & $0.072 \pm 0.075$ \\
    100 & 3000 & $0.657 \pm 0.061$ & $0.999 \pm 0.001$ & $0.219 \pm 0.038$ \\
    250 & 7500 & $0.688 \pm 0.043$ & $1.000$ & $0.089 \pm 0.102$ \\
    500 & 15000 & $0.727 \pm 0.044$ & $1.000$ & $0.503 \pm 0.067$ \\
    1000 & 30000 & $0.728 \pm 0.005$ & $1.000$ & $0.724 \pm 0.164$ \\
    5000 & 150000 & $0.814 \pm 0.048$ & $1.000$ & $0.830 \pm 0.058$ \\
    \bottomrule
  \end{tabular}
\end{table}

Two patterns emerge. First, MCC climbs smoothly from 0.56 to 0.81 with no phase transition, indicating that latent identifiability degrades gradually at low $N$. Second, $X_Z$@top3 exhibits a sharper transition: near zero for $N/D \leq 50$, rising through 0.50 around $N/D{=}500$, and reaching 0.83 at $N/D{=}5000$. Support recovery is more data-hungry than latent recovery, consistent with Theorem~\ref{thm:finite}'s prediction that the sample requirement scales with the signal-to-noise ratio $(\sigma_\epsilon^2 + \sigma_{r,\max}^2) / \sigma_{\min}^2$, which is a fixed population-level quantity that determines the $N$ threshold for reliable recovery.

Regime accuracy saturates at $N/D \geq 25$ (above 0.998), confirming that regime-label signal is not the bottleneck. Note that temporal subsampling reduces not only $N$ but also the density of temporal transitions, which may additionally degrade the temporal prior's nonstationarity signal at very low $N/D$.

\subsection{Case study: regime definition determines the scientific question}
\label{app:synthetic:regime_def}

The regime label encodes the scientific question that MOSAIC answers. We demonstrate this on the synthetic benchmark by applying two regime definitions to the same data, each targeting a different physical contrast, and showing that MOSAIC returns a different (and correct) regime-associated factor in each case.

\begin{figure}[h]
\centering
\includegraphics[width=\textwidth]{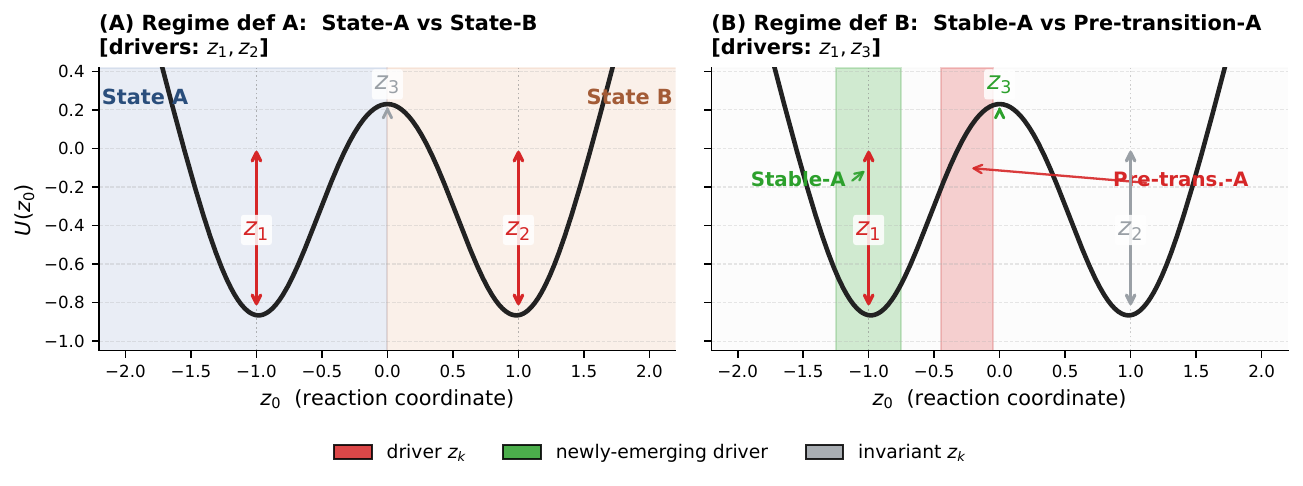}\\[2pt]
\includegraphics[width=\textwidth]{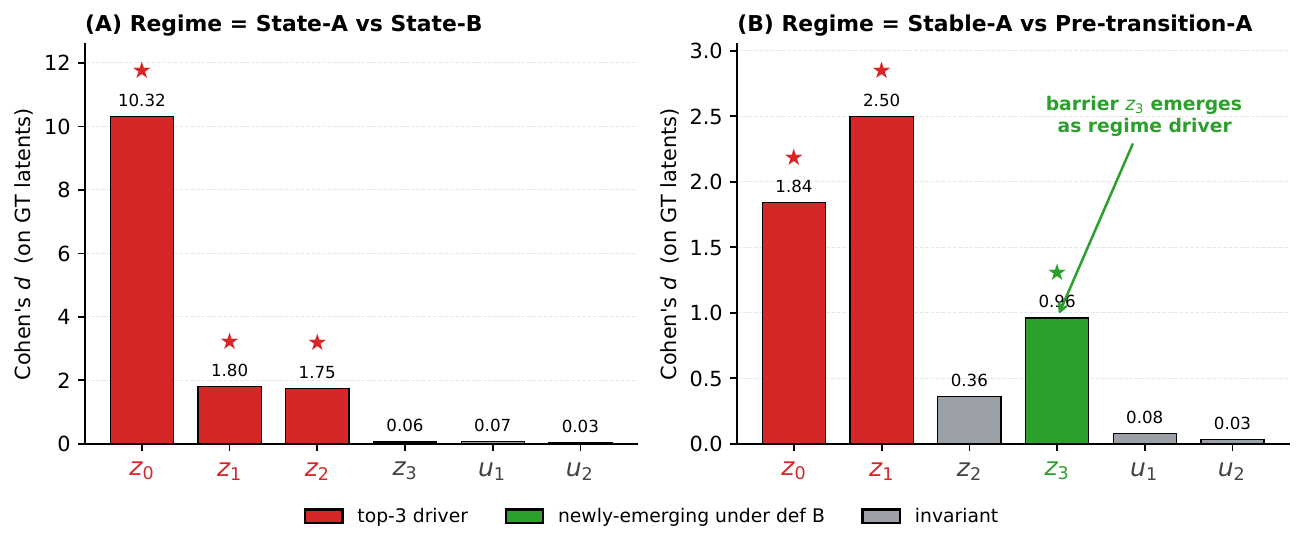}
\caption{Regime definition determines the scientific question. \textbf{Top row}: the two regime partitions overlaid on the same parametric-coupling double-well, with $z_1$ (left-well depth), $z_2$ (right-well depth), and $z_3$ (barrier height) annotated in each panel; colour marks each factor as regime-varying (red), newly regime-associated (green), or invariant (gray) under the regime shown. (A) Def A splits the reaction coordinate into the left well (State A) and right well (State B); both well-depth factors $z_1$ and $z_2$ are regime-varying. (B) Def B splits the left well temporally into Stable-A (near the well bottom) and Pre-transition-A (approaching the barrier); the barrier-height factor $z_3$ becomes most regime-associated while $z_2$ becomes invariant. \textbf{Bottom row}: Cohen's $d$ of each ground-truth latent confirms the shift. The same latent system yields a different regime-association set because the scientific question encoded in the regime label has changed.}
\label{fig:regime_def}
\end{figure}

\textbf{Definition A: which latent modules distinguish State A from State B?}
Regime labels are assigned by well occupancy (State A vs.\ State B). Under this definition, MOSAIC recovers the full module structure (MCC $\approx$ 0.85, $X_Z$@top3 $\approx$ 0.88) and correctly identifies the reaction-coordinate factors as regime-varying. This is the standard setting reported in Section~\ref{sec:synthetic}.

\textbf{Definition B: what signals precede a transition from A to B?}
We redefine regimes as stable-A (particle in well A with no transition within 50 frames) versus pre-transition-A (particle in well A but about to transition within 50 frames). This asks a fundamentally different question: not ``what differs between states'' but ``what predicts an upcoming transition.''

Under Definition B, MCC drops from 0.85 to $0.565 \pm 0.042$ and $X_Z$@top3 drops from $\sim$0.88 to $0.022 \pm 0.039$ (Table~\ref{tab:regime_def_b}). This degradation is physically expected: the pre-transition signal depends on future behavior, but MOSAIC's temporal prior conditions only on past observations. The temporal prior can detect distributional differences between regimes that are visible in the current and past frames; it cannot anticipate future transitions that have not yet manifested in observable dynamics.

\begin{table}[h]
  \caption{Regime Definition B (pre-transition labeling) on the synthetic benchmark, 3 seeds. Definition A (well-occupancy) achieves MCC $\approx$ 0.85, $X_Z$@top3 $\approx$ 0.88.}
  \label{tab:regime_def_b}
  \centering
  \small
  \begin{tabular}{rcccc}
    \toprule
    Seed & Regime Acc & MCC & Z@top3 & $X_Z$@top3 \\
    \midrule
    0 & 0.901 & 0.528 & 0/3 & 0.000 \\
    42 & 0.931 & 0.609 & 2/3 & 0.067 \\
    123 & 0.936 & 0.559 & 1/3 & 0.000 \\
    \midrule
    Mean & 0.923 & 0.565 & $\sim$1/3 & 0.022 \\
    \bottomrule
  \end{tabular}
\end{table}

\textbf{Takeaway.}
The contrast between Definitions A and B clarifies the scope of MOSAIC's regime-association concept: it identifies factors whose distributions differ most across the user-specified regimes, not factors that predict future regime transitions. This is analogous to the RNA cross-regime result (Section~\ref{sec:rna}), where switching the regime definition from closing-pair contacts to outer-stem contacts shifts the top regime-associated factor from Loop to Stem. In both cases, MOSAIC faithfully answers the question encoded in the regime label. Users should choose regime definitions that reflect the scientific contrast of interest.

\subsection{Unsupervised regime discovery via k-means}
\label{app:synthetic:auto_regime}

MOSAIC requires binary regime labels. We test whether unsupervised regime discovery can substitute for ground-truth labels by running k-means ($k{=}2$) on the raw observations to generate pseudo-regime labels, then training MOSAIC with these labels (3 seeds).

\begin{table}[h]
  \caption{Auto-regime on the synthetic benchmark: k-means pseudo-labels vs.\ ground-truth labels. NMI and agreement are between pseudo and GT labels.}
  \label{tab:auto_regime}
  \centering
  \small
  \begin{tabular}{rccccc}
    \toprule
    Seed & NMI & Agreement & MCC & Z@top3 & $X_Z$@top3 \\
    \midrule
    0 & 0.959 & 99.5\% & 0.877 & 3/3 & 0.717 \\
    42 & 0.959 & 99.5\% & 0.858 & 2/3 & 0.600 \\
    123 & 0.959 & 99.5\% & 0.823 & 2/3 & 0.600 \\
    \midrule
    Mean & & & $0.853 \pm 0.027$ & & $0.639 \pm 0.068$ \\
    \bottomrule
  \end{tabular}
\end{table}

On the synthetic benchmark, k-means nearly perfectly recovers the ground-truth regime (NMI\,=\,0.959, 99.5\% agreement), and MOSAIC trained on pseudo-labels achieves MCC\,=\,0.853, matching the supervised result ($\sim$0.85). $X_Z$@top3 is lower (0.64 vs.\ $\sim$0.88), possibly because minor label noise at regime boundaries degrades the sparsity penalty signal. This validates unsupervised regime discovery on well-separated synthetic data; real-data domains with less clear-cut regime structure may require iterative label refinement.

\section{Theory Proofs and Finite-Sample Rates}
\label{app:theory}

\subsection{Population Identifiability Proofs}
\label{app:proofs_population}

\textbf{Proof of Proposition~\ref{prop:no_fp} (Main-Effect Identifiability Without False Positives).}
The ANOVA decomposition $g(\vz) = g_0 + \sum_j g_j(z_j) + r(\vz)$ is the unique $L^2(\mu)$-orthogonal expansion satisfying $\E_{\mu_j}[g_j(z_j)] = 0$ and $r$ orthogonal to every univariate function of a single $z_j$~\citep{hoeffding1992class}. Hence the population MSE decomposes as
$\E_{\mu}\|\vx - \hat g(\vz)\|^2 = \E\|\boldsymbol{\epsilon}\|^2 + \sum_j \E_{\mu_j}\|g_j(z_j) - \hat f_j(z_j)\|^2 + \E_{\mu}\|r(\vz)\|^2$,
with no cross terms. The sum is minimized independently in each $\hat f_j$ by $\hat f_j = g_j$, and the residual $\E_{\mu}\|r\|^2$ is irreducible by any additive model. Therefore $\hat f_j^{(i)}$ is non-constant iff $g_j^{(i)}$ is non-constant, so $\hat{\mathcal{S}}_j = \mathcal{S}_j \subseteq \mathcal{S}_j^{\mathrm{full}}$, where $\mathcal{S}_j^{\mathrm{full}}$ denotes the full support of $g$ with respect to $z_j$. \hfill $\square$

\textbf{Proof of Theorem~\ref{thm:module} (Module Structure Identifiability).}
Under the theorem's hypothesis both models have $\rho = 0$, so $g(\vz) = g_0 + \sum_j g_j(z_j)$ and $\tilde g(\vz) = \tilde g_0 + \sum_k \tilde g_k(\tilde z_k)$ are exactly additive.

\textbf{Step 1 (Latent identifiability).} By Assumption~A2 and the identifiability result of \citet{yao2022temporally} (Theorem~1), the latent representations of the two models are related by $\tilde{z}_k = \phi_{\pi(k)}(z_{\pi(k)})$ for all $k \in [n]$, where $\pi$ is a permutation and each $\phi_j: \real \to \real$ is a smooth diffeomorphism.

\textbf{Step 2 (Additive structure constrains the mapping).} At the population optimum the two models have the same conditional mean, so $\sum_{j=1}^{n} g_j(z_j) + g_0 = \sum_{k=1}^{n} \tilde{g}_k(\tilde{z}_k) + \tilde{g}_0$. Substituting $\tilde{z}_k = \phi_{\pi(k)}(z_{\pi(k)})$ and re-indexing via $k = \pi^{-1}(j)$ gives two additive decompositions of the same function of $\vz$. After absorbing component means into the intercept, the centered Hoeffding decomposition under the product reference measure is unique. Therefore $g_j(z_j) = \tilde{g}_{\pi^{-1}(j)}(\phi_j(z_j)) + c_j$ for all $j \in [n]$, where $c_j \in \real^D$ are constants satisfying $\sum_j c_j = \tilde{g}_0 - g_0$.

\textbf{Step 3 (Support preservation).} Examining coordinate $i$: $g_j^{(i)}(z_j) = \tilde{g}_{\pi^{-1}(j)}^{(i)}(\phi_j(z_j)) + c_j^{(i)}$. If $i \in \mathcal{S}_j$ (i.e., $g_j^{(i)}$ is non-constant), then since $\phi_j$ is a diffeomorphism, the composition $\tilde{g}_{\pi^{-1}(j)}^{(i)} \circ \phi_j$ is also non-constant, implying $i \in \tilde{\mathcal{S}}_{\pi^{-1}(j)}$. The converse follows by applying the same argument with $\phi_j^{-1}$. Therefore $\mathcal{S}_j = \tilde{\mathcal{S}}_{\pi^{-1}(j)}$ for all $j$. \hfill $\square$

\textbf{Proof of Corollary~\ref{cor:driver} (Regime-Association Identifiability).}
By Step~1 of Theorem~\ref{thm:module}, $\tilde{z}_k = \phi_{\pi(k)}(z_{\pi(k)})$ where each $\phi_j$ is invertible. The KL divergence is invariant under invertible transformations: $\mathrm{KL}(p(\tilde{z}_k \mid c{=}0) \| p(\tilde{z}_k \mid c{=}1)) = \mathrm{KL}(p(z_{\pi(k)} \mid c{=}0) \| p(z_{\pi(k)} \mid c{=}1))$. Therefore the ranking, and hence the argmax set, is preserved under $\pi^{-1}$; if the maximizer is unique, $\tilde{j}^* = \pi^{-1}(j^*)$. \hfill $\square$

\subsection{Product Measure Versus Data Distribution}
\label{app:product_measure_gap}

Proposition~\ref{prop:no_fp} defines main-effect identifiability under the product reference measure $\mu = \prod_j \mu_j$, while the Stage~2 decoder is optimized under the empirical distribution induced by the frozen encoder. Two observations bridge this gap.

First, the ANOVA main effects $\{g_j\}$ are population-level objects uniquely determined by $g$ and $\mu$; they serve as \emph{targets} for the additive decoder rather than quantities that the decoder defines.

Second, by the time Stage~2 begins, Stage~1 identifiability (Assumption~A2) has already pushed the encoder toward a representation whose marginal factorizes across dimensions up to permutation and component-wise transformation. Any residual statistical dependence among the encoded coordinates is absorbed by the interaction term $r(\vz)$ in the ANOVA expansion, leaving the per-component main effects $g_j$ unchanged. The practical consequence is that the Stage~2 additive decoder targets the correct main effects to the extent that the encoder's posterior marginals approximate independent factors, a condition that improves as Stage~1 identifiability tightens.

We emphasize that this product-versus-empirical gap is not specific to MOSAIC: all temporal CRL methods that invoke nonlinear ICA identifiability face the same discrepancy between the theoretical reference measure and the learned encoder distribution. Proposition~\ref{prop:no_fp} is a statement about the population optimum under $\mu$; the finite-sample gap between $\mu$ and the empirical encoder distribution is part of the overall estimation error characterized empirically in Appendices~\ref{app:synthetic:nd} and~\ref{app:synthetic:interactions}.

\subsection{Finite-Sample Rates for Sparse Additive Estimators}
\label{app:finite_sample}

The population identifiability results above hold for any smooth estimator. The two results in this subsection, together with their proofs, instantiate finite-sample guarantees for one specific estimator class: the classical sparse additive model with B-spline main effects and group-lasso regularization~\citep{ravikumar2009sparse}. They serve as a reference point for what is achievable in the additive ANOVA framework, but do not directly govern the MLP-based decoder used by MOSAIC. The empirical finite-sample behavior of MOSAIC's implementation is characterized in Appendices~\ref{app:synthetic:nd},~\ref{app:synthetic:interactions}, and~\ref{app:rho_calibration}; the design rationale and empirical comparison between entropy and group-lasso regularization are given in Appendix~\ref{app:entropy_vs_grouplasso}.

\begin{proposition}[Finite-Sample Main-Effect Estimation]
\label{prop:finite_sample_main_effect}
Let $\hat f_j^{(N)}$ be the additive fit to $N$ i.i.d.\ samples using a per-component B-spline basis with $m_N=O(N^{1/5})$. Under standard smoothness and sub-Gaussian design conditions~\citep{ravikumar2009sparse},
\begin{equation}
\E_{\mu}\bigl\|\hat f_j^{(N),(i)} - g_j^{(i)}\bigr\|^2 = O(N^{-4/5}) \quad \forall j \in [n],\, i \in [D],
\end{equation}
with constants depending on smoothness and on $\sigma_\epsilon^2+\E_{\mu}\|r(\vz)\|^2$. This is the standard univariate spline rate for nonparametric regression. For time series, $N$ refers to independent windows, or to the effective sample size under mixing.
\end{proposition}

\textbf{Proof of Proposition~\ref{prop:finite_sample_main_effect}.}
Apply Theorem~3 of~\citet{ravikumar2009sparse} to each output channel $i \in [D]$ separately. Under sub-Gaussian design and the standard smoothness assumption that each $g_j^{(i)}$ lies in a second-order Sobolev class, spline smoothing with $m_N = O(N^{1/5})$ basis functions attains the univariate minimax rate $\E_{\mu}\|\hat f_j^{(N),(i)} - g_j^{(i)}\|^2 = O(N^{-4/5})$. The $L^2(\mu)$-orthogonality of the ANOVA decomposition makes the population MSE separable across channels. Each channel independently achieves the $O(N^{-4/5})$ rate, so the per-channel convergence exponent is unaffected by $D$; the total squared error summed over all $D$ channels scales as $O(D \cdot N^{-4/5})$, but the rate exponent remains $-4/5$ per channel. The interaction residual $r(\vz)$ contributes to the effective noise variance $\sigma_\epsilon^2 + \E_{\mu}\|r\|^2$ by the same orthogonality argument used in the proof of Theorem~\ref{thm:finite}. \hfill $\square$

\begin{theorem}[Finite-Sample Support Recovery, Group-Lasso Variant]
\label{thm:finite}
Let $\sigma_{\min}^2 = \min_{j, i \in \mathcal{S}_j} \mathrm{Var}_{\mu_j}[g_j^{(i)}(z_j)]$ denote the minimum main-effect signal, let $\sigma_{r,\max}^2 = \max_{i \in [D]} \mathrm{Var}_{\mu}[r^{(i)}(\vz)]$ denote the maximum per-channel interaction variance, and let $N$ denote the number of independent samples (or effective sample size under mixing). Under Assumption~A1$'$, sub-Gaussian design, and the sparse-additive incoherence and restricted eigenvalue conditions of \citet{ravikumar2009sparse} applied per output channel with $m_N = O(N^{1/5})$ B-spline basis functions, the group-lasso-regularized additive decoder recovers $\mathcal{S}_j$ for all $j \in [n]$ with high probability provided
\begin{equation}
N \;\gtrsim\; C(s_{\max}, n, m_N) \;\cdot\; \frac{\sigma_\epsilon^2 + \sigma_{r,\max}^2}{\sigma_{\min}^2},
\end{equation}
where $s_{\max} = \max_{i \in [D]} |\{j : i \in \mathcal{S}_j\}|$ is the maximum number of active factors per output channel and $C(s_{\max}, n, m_N)$ collects the sparsity- and basis-dependent constants from \citet{ravikumar2009sparse}. Under A1$'$, $s_{\max} = 1$ for all primary-support channels.
\end{theorem}

The $D$ output channels share the same $N$ design points (the latent samples $\hat{\vz}_t$), so the sample complexity does not scale with $D$; it is governed by the per-channel signal-to-noise ratio $\sigma_{\min}^2 / (\sigma_\epsilon^2 + \sigma_{r,\max}^2)$ and the sparsity level $s_{\max}$. Higher-order interactions inflate $\sigma_{r,\max}^2$ and therefore the sample requirement, while stronger main-effect signals $\sigma_{\min}^2$ relax it.

\textbf{Proof sketch of Theorem~\ref{thm:finite}.}
We apply the support recovery analysis of \citet{ravikumar2009sparse} per output channel. For each channel $i \in [D]$, write the additive regression as $x_i = \sum_j g_j^{(i)}(z_j) + \eta_i$, where $\eta_i = r^{(i)}(\vz) + \epsilon_i$ is the effective residual. Two properties make this reduction valid. First, $L^2(\mu)$-orthogonality of the ANOVA decomposition ensures $\E_\mu[\eta_i \cdot h(z_j)] = \E_\mu[\epsilon_i] \cdot \E_\mu[h(z_j)] = 0$ for any univariate $h$ with $\E_{\mu_j}[h] = 0$, so $\eta_i$ is uncorrelated with the additive regressors under the product measure. Second, because $r^{(i)}$ is a smooth bounded function of $\vz$ and $\epsilon_i$ is Gaussian, the sum $\eta_i$ is sub-Gaussian with parameter bounded by $\sigma_\epsilon^2 + \mathrm{Var}_\mu[r^{(i)}(\vz)]$, provided the latent distribution $\mu$ has sub-Gaussian marginals (a condition satisfied by the Laplace transition prior). We note that $\eta_i$ is not independent of $\vz$, since $r^{(i)}$ is a deterministic function of $\vz$; the application of \citet{ravikumar2009sparse} therefore requires their fixed-design variant or an additional assumption that the empirical covariance of the B-spline basis concentrates despite the dependence, which holds under the sub-Gaussian design condition stated in the theorem. Applying the per-channel support recovery bound, the worst-case channel determines the overall sample requirement, giving the stated bound with $\sigma_{r,\max}^2 = \max_i \mathrm{Var}_\mu[r^{(i)}(\vz)]$. At $\rho = 0$, $r \equiv 0$ and the bound reduces to the standard sparse additive model guarantee with Gaussian noise alone. Assumption~A1$'$ ensures $s_{\max} = 1$ for all primary-support channels, simplifying the incoherence requirement. Proposition~\ref{prop:partial_overlap} below records the hard-assignment behavior when supports overlap. \hfill $\square$

\subsection{Shared-Support Assignment and Entropy-Variant Lemmas}
\label{app:proofs_entropy}

\textbf{Proof of Proposition~\ref{prop:partial_overlap}.}
(i) By $L^2(\mu)$-orthogonality of the ANOVA decomposition applied per channel, the population MSE for channel $i$ decomposes as $\sum_{j': i \in \mathcal{S}_{j'}} \E_{\mu_{j'}}\|g_{j'}^{(i)} - \hat f_{j'}^{(i)}\|^2 + \text{const}$, minimized independently in each summand by $\hat f_{j'}^{(i)} = g_{j'}^{(i)}$. The argument is identical to Proposition~\ref{prop:no_fp}, applied channel-wise and in parallel to all factors whose support contains $i$.

(ii) For any factor $\ell \notin \mathcal{J}_i$, $g_\ell^{(i)} \equiv 0$, so any valid influence score has $A_{i\ell}=0$. If the in-support maximum is strictly positive and tie-breaking does not prefer zero-score factors, the argmax cannot select $\ell\notin\mathcal{J}_i$. Among overlapping factors, the rule selects the largest observed influence score by definition. The finite contrast $|f_m(+1)_i-f_m(-1)_i|$ used in MOSAIC is one such score when the contrast is non-degenerate; integrated or quantile contrasts can replace it for symmetric responses. \hfill $\square$

\textbf{Scores for symmetric responses.}
The finite contrast $|f_j(+1)_i - f_j(-1)_i|$ vanishes for even main effects such as $g_j^{(i)}(z_j) = a z_j^2$. Two alternative scores handle this case. The \emph{variance score} $A_{ij}^{\mathrm{var}} = \mathrm{Var}_{z_j \sim \hat\mu_j}[f_j(z_j)_i]$, estimated from the empirical encoder marginal, is nonzero whenever $f_j^{(i)}$ is non-constant regardless of symmetry. The \emph{range score} $A_{ij}^{\mathrm{range}} = \max_{z_j \in [\hat q_{0.05}, \hat q_{0.95}]} f_j(z_j)_i - \min_{z_j \in [\hat q_{0.05}, \hat q_{0.95}]} f_j(z_j)_i$ provides a robust finite-sample alternative. In our experiments, the monotonic nonlinear mixing functions (Section~\ref{sec:synthetic}) and the physical distance features (Section~\ref{sec:rna}) produce non-symmetric main effects, so the default finite contrast suffices; the variance and range scores are provided for generality.

\textbf{Implementation-relevant statements for the entropy penalty.}
The following lemma and proposition describe properties of the entropy-regularized implementation. They do not supply a finite-sample support-recovery theorem (that is an open problem; see Appendix~\ref{app:entropy_vs_grouplasso}), but they record the weaker facts that entropy biases the decoder toward concentrated influence columns and that exact top-$k$ support recovery follows under an influence-gap condition on the estimated column.

\begin{lemma}[Entropy Penalizes Diffuse Influence Columns]
\label{lem:entropy_concentration}
Let $a \in \mathbb{R}^D_{\ge 0}$ be a nonzero influence column and define $p_i = a_i / \sum_{\ell=1}^D a_\ell$. Then
\[
0 \le H(p) \le \log D,
\]
where the lower bound is attained by one-hot columns and the upper bound is attained by the uniform column. Thus, for fixed column magnitude, minimizing $H(p)$ biases the decoder toward concentrated rather than diffuse influence columns.
\end{lemma}

\begin{proof}
This is the standard entropy bound on a discrete distribution over finite $D$ elements. The minimum entropy is attained by any point mass, and the maximum entropy is attained by the uniform distribution.
\end{proof}

\begin{proposition}[Top-$k$ Support Consistency from Influence Separation]
\label{prop:topk_support_consistency}
For a latent factor $j$, let $A^*_{:,j}$ be the population influence column induced by the additive main effect $g_j$, and let $\mathcal{S}_j$ be its support with $k_j = |\mathcal{S}_j|$. Assume an influence gap
\[
\Delta_j = \min_{i \in \mathcal{S}_j} A^*_{ij} - \max_{i \notin \mathcal{S}_j} A^*_{ij} > 0.
\]
If an estimated additive decoder satisfies
\[
\|\hat A_{:,j} - A^*_{:,j}\|_\infty < \Delta_j / 2,
\]
then the top-$k_j$ entries of $\hat A_{:,j}$ recover $\mathcal{S}_j$ exactly.
\end{proposition}

\begin{proof}
For any $i \in \mathcal{S}_j$ and $\ell \notin \mathcal{S}_j$, the influence gap definition gives $A^*_{ij} \ge \min_{i' \in \mathcal{S}_j} A^*_{i'j} = \max_{\ell' \notin \mathcal{S}_j} A^*_{\ell' j} + \Delta_j \ge A^*_{\ell j} + \Delta_j$. Combined with $\|\hat A_{:,j} - A^*_{:,j}\|_\infty < \Delta_j/2$:
\[
\hat A_{ij} > A^*_{ij} - \Delta_j/2 \ge A^*_{\ell j} + \Delta_j - \Delta_j/2 = A^*_{\ell j} + \Delta_j/2 > \hat A_{\ell j}.
\]
Thus every true-support coordinate has strictly larger estimated influence than every false-support coordinate, so selecting the top $k_j$ entries recovers $\mathcal{S}_j$.
\end{proof}

\section{Entropy Sparsity Penalty: Mechanism and Design}
\label{app:entropy_vs_grouplasso}

The main text introduces the entropy penalty (Eq.~\ref{eq:entropy_sparsity}) as MOSAIC's mechanism for recovering module structure, but defers the question of why entropy works and why alternative sparsity penalties do not. This appendix gives a mechanistic account in four parts: (i) what entropy regularizes once the influence column is normalized, (ii) why this design is compatible with reconstruction in a way that magnitude-based penalties are not, (iii) how entropy interacts with the alive-mask to handle over-capacity, and (iv) the role of entropy relative to the population identifiability results of Section~\ref{sec:identifiability}. An important terminological distinction: Theorem~\ref{thm:finite} analyzes the \emph{basis-function} group-lasso of \citet{ravikumar2009sparse}, which groups B-spline coefficients for each (latent, channel) pair to perform per-channel variable selection. The column-level penalty compared empirically below, $\sum_j \|A_{:,j}\|_2$, groups across output channels for a given latent and targets a different object: within-column concentration of the influence matrix. Theorem~\ref{thm:finite} serves as a reference point for what the basis-function group-lasso achieves in the ANOVA framework; the empirical comparison below evaluates whether the column-level penalty or entropy is more effective at recovering module structure from MOSAIC's additive decoder.

\textbf{What entropy regularizes.}
Fix a factor index $j$ with influence column $A_{:,j} \in \mathbb{R}^D_{\geq 0}$ and let $p_{:,j} = A_{:,j} / \|A_{:,j}\|_1$ denote its normalization onto the probability simplex. The penalty
\begin{equation}
R_{\text{ent}}(A) = \frac{1}{|\mathcal{J}_{\text{alive}}|} \sum_{j \in \mathcal{J}_{\text{alive}}} H(p_{:,j})
\end{equation}
is a function of the \emph{shape} of $A_{:,j}$, not its total magnitude. Concretely, $R_{\text{ent}}$ is invariant under positive rescaling of any single column: $A_{:,j} \mapsto c \cdot A_{:,j}$ leaves $p_{:,j}$ and hence $H(p_{:,j})$ unchanged for any $c > 0$. Lemma~\ref{lem:entropy_concentration} bounds $H(p_{:,j}) \in [0, \log D]$, with the lower bound attained at one-hot columns and the upper bound at uniform columns. Minimizing $R_{\text{ent}}$ therefore biases each column toward concentration on a small subset of observed variables, the support set $\mathcal{S}_j$ of Definition~\ref{def:support}.

\textbf{Decoupling from reconstruction.}
The shape-magnitude split is what makes entropy compatible with reconstruction. The Stage~2 objective $\cL_{\text{Stage2}} = \cL_{\text{ELBO}} + \lambda(t) \cdot \cL_{\text{sparse}}$ asks each $f_j$ to carry enough signal that $\hat{\vx} = \sum_j f_j(z_j) + \mathbf{b}$ reconstructs $\vx$. This pins down the magnitude $\|A_{:,j}\|_1$, but says nothing about how that magnitude is distributed across the $D$ observation channels. Entropy operates exactly on this remaining degree of freedom: it rearranges mass within a column without contesting the column's overall scale. Reconstruction and entropy therefore act on orthogonal degrees of freedom, and an additive decoder can satisfy both: $\|A_{:,j}\|_1$ is set by the strength of the main effect $g_j$, and the shape of $A_{:,j}$ collapses onto $\mathcal{S}_j$.

This decoupling is the essential property a sparsity penalty needs in order not to fight reconstruction, and it is precisely what magnitude-based penalties lack. Element-wise $\ell_1$ ($\sum_{i,j} |A_{ij}|$) and group-lasso ($\sum_j \|A_{:,j}\|_2$) both penalize magnitude. They reduce in the same way regardless of whether the model (a) drops zero-support entries while preserving in-support magnitude, which is what we want, or (b) shrinks the entire column toward zero, which kills the latent. Under (a) reconstruction stays good but the penalty does not fall as fast; under (b) the penalty falls but reconstruction degrades. The optimization typically lands at a compromise, alive-but-diffuse columns with reduced magnitude, which violates module structure. Consider two columns with identical $\ell_1$ norm $\alpha$: $A_{:,j}^{(\text{conc})} = (\alpha, 0, \ldots, 0)$ and $A_{:,j}^{(\text{diff})} = (\alpha/D, \ldots, \alpha/D)$. Both carry the same total influence magnitude, so reconstruction treats them as equivalent alternatives. However, the column-level group-lasso penalty $\|A_{:,j}\|_2$ is $\alpha$ for the concentrated column and $\alpha/\sqrt{D}$ for the diffuse column, so it \emph{prefers} the diffuse solution. Entropy gives $H = 0$ for the concentrated column and $H = \log D$ for the diffuse column, correctly preferring concentration. This illustrates the fundamental misalignment: magnitude-based penalties reward spreading mass across channels, while entropy rewards concentrating it.

\textbf{Empirical comparison on the synthetic benchmark.}
We compare entropy and group-lasso head-to-head on the synthetic energy-landscape benchmark with $z_{\dim}=12$, sweeping the group-lasso coefficient across $\lambda \in \{10, 50, 100, 1000\}$ to identify its best operating point. Entropy uses its main-text setting $\lambda=50$. Table~\ref{tab:entropy_vs_gl} reports MCC for latent identifiability, $X_Z$@top3 for support recovery, and top-3 mass for column concentration.

\begin{table}[h]
\centering
\small
\caption{Entropy vs group-lasso on the synthetic benchmark ($z_{\dim}=12$, 5 seeds). Group-lasso is swept across four coefficients to identify its best operating point. MCC is unchanged across all variants because Stage~1 identifiability is independent of the Stage~2 sparsity choice.}
\label{tab:entropy_vs_gl}
\begin{tabular}{l r c c c}
\toprule
Variant & $\lambda$ & MCC & $X_Z$@top3 & Top-3 mass \\
\midrule
\textbf{Entropy} & \textbf{50} & $\mathbf{0.912 \pm 0.009}$ & $\mathbf{0.978 \pm 0.044}$ & $\mathbf{0.685 \pm 0.028}$ \\
\midrule
Group-lasso      & 10   & $0.912 \pm 0.009$ & $0.822 \pm 0.113$ & $0.632 \pm 0.026$ \\
Group-lasso      & 50   & $0.912 \pm 0.009$ & $0.867 \pm 0.130$ & $0.643 \pm 0.020$ \\
Group-lasso      & 100  & $0.912 \pm 0.009$ & $0.911 \pm 0.083$ & $0.642 \pm 0.015$ \\
Group-lasso      & 1000 & $0.912 \pm 0.009$ & $0.556 \pm 0.199$ & $0.683 \pm 0.036$ \\
\bottomrule
\end{tabular}
\end{table}

\begin{figure}[h]
\centering
\includegraphics[width=0.95\linewidth]{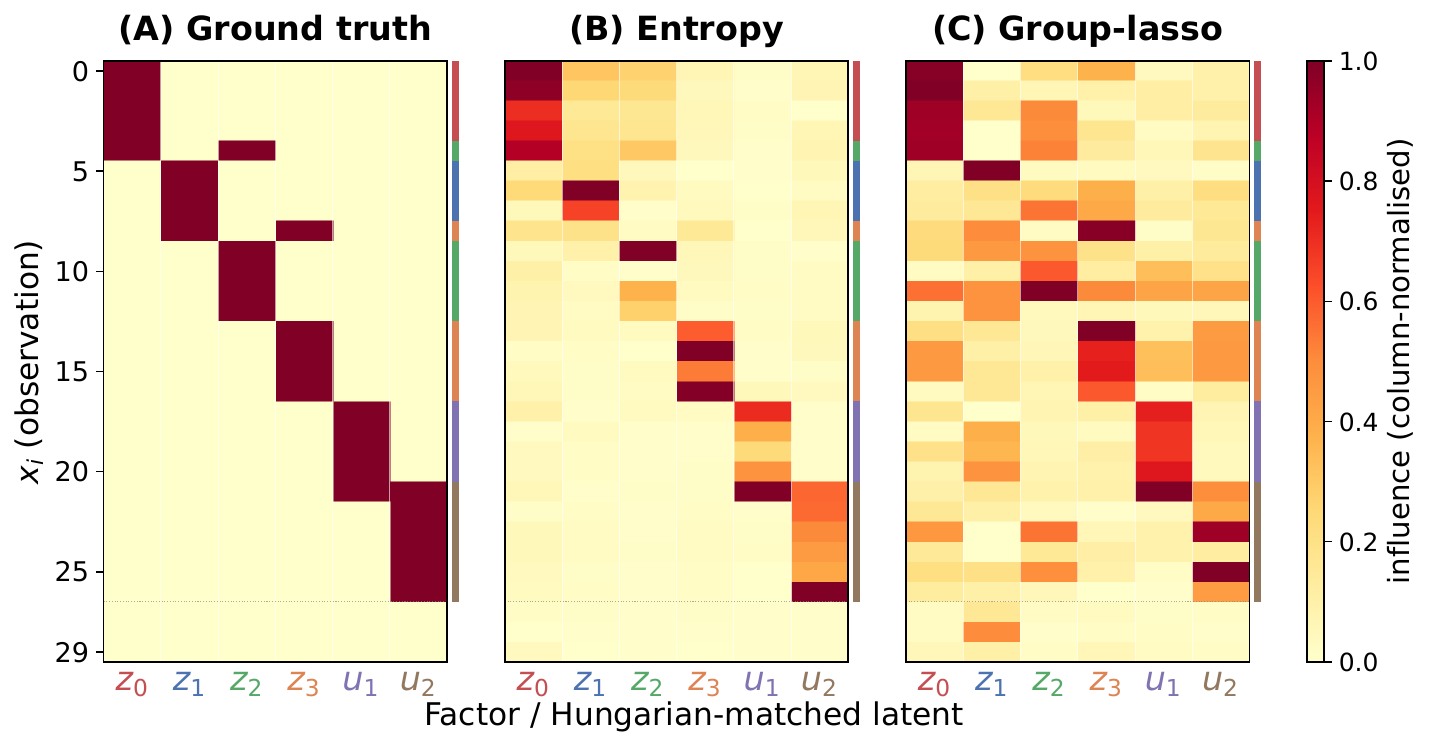}
\caption{Influence matrices on the synthetic benchmark, column-normalized. \textbf{(A)} Ground-truth support structure: each true factor maps to a disjoint block of observed variables, with channels $27$--$29$ as pure noise. \textbf{(B)} Entropy variant ($\lambda=50$): the diagonal block structure is recovered with concentrated mass on the true support and minimal off-support leakage. \textbf{(C)} Group-lasso variant at its best operating point ($\lambda=100$): the diagonal is recovered but every column carries non-trivial off-support mass, illustrating the magnitude-not-shape character of the penalty. Both panels use the same Hungarian alignment of learned latents to ground truth.}
\label{fig:entropy_vs_gl}
\end{figure}

Three patterns emerge. First, MCC is identical at $0.912 \pm 0.009$ across all variants. The choice of sparsity penalty does not affect Stage~1 latent identifiability; that signal comes entirely from the temporal prior, and the penalty only shapes the Stage~2 decoder. Second, at its best setting $\lambda=100$, group-lasso reaches $X_Z$@top3 of $0.911$, still below entropy's $0.978$ and with substantially higher seed variance ($\pm 0.083$ vs $\pm 0.044$). Figure~\ref{fig:entropy_vs_gl}(C) shows why: at this $\lambda$ the group-lasso decoder concentrates the diagonal entries correctly but leaves visible off-support mass in nearly every column, particularly on $z_2$, $z_3$, and $u_1$. This residual mass drives $X_Z$ variance through unstable top-$k$ selection. Entropy at $\lambda=50$ produces the visibly cleaner diagonal of Figure~\ref{fig:entropy_vs_gl}(B). Third, pushing group-lasso to $\lambda=1000$ collapses $X_Z$@top3 to $0.556 \pm 0.199$ even though MCC is unchanged. This is the magnitude-shrinkage failure mode anticipated earlier: with strong enough magnitude pressure, group-lasso drives column norms toward zero rather than concentrating mass within columns, and the resulting low-magnitude columns no longer carry a clean support signal. Notably, at $\lambda=1000$ the top-3 mass ($0.683$) matches entropy's ($0.685$), yet $X_Z$@top3 collapses: the column does concentrate on three channels, but those channels are no longer the true support. Once magnitude shrinks below the noise floor, top-$k$ selection picks up spurious entries. Entropy is invariant to this collapse by construction, since it only acts on normalized column shapes.

\textbf{Why the alive mask is not optional.}
Scale invariance has a flip side: entropy alone cannot kill a latent. Multiplying a column by $c \to 0^+$ leaves $H(p_{:,j})$ unchanged, so a near-dead latent looks the same to $R_{\text{ent}}$ as an active one. When $z_{\dim} > n_{\text{true}}$, this would force the entropy penalty to concentrate excess latents onto some channel, producing spurious modules anchored to whichever channel the optimizer happened to pick first.

The alive mask $\mathcal{J}_{\text{alive}} = \{j : \sum_i A_{ij} > 0.01 \cdot \max_k \sum_i A_{ik}\}$ supplies the missing on/off mechanism. Columns whose total magnitude has collapsed are excluded from the entropy sum, so the optimizer is free to leave them at zero rather than forced to spike them. The two mechanisms partition the work cleanly: the alive mask handles the binary alive/dead decision (driven by reconstruction, since a latent stays alive only if it carries reconstruction signal), and entropy handles the shape decision conditional on being alive. Appendix~\ref{app:synthetic:zdim} quantifies the consequences when this partition is stressed: at $z_{\dim} = 20$ the alive-mask threshold is near its operating limit, support recovery becomes seed-dependent, and $X_Z$@top3 drops from $0.978$ at $z_{\dim} = 12$ to $0.733$.

\textbf{Relation to population identifiability.}
Proposition~\ref{prop:no_fp} is unconditional on the sparsity penalty: at the population optimum the additive fit recovers each main effect $g_j$ exactly, so $\hat{\mathcal{S}}_j = \mathcal{S}_j$ holds without any sparsity regularization. Entropy is therefore not part of the identifiability machinery. Its role is finite-sample readability. With finite $N$, the estimated $\hat{f}_j$ carries small but non-zero residual mass on channels outside $\mathcal{S}_j$, due to a combination of optimization noise and finite-sample bias. Without an explicit concentration mechanism, the influence matrix $A$ contains many small entries everywhere, and any thresholding rule (top-$k$, ratio-gap, $\arg\max$) becomes seed-dependent. Entropy compresses this spurious mass back toward zero, making $\mathcal{S}_j$ readable from a finite-sample $A$. Proposition~\ref{prop:topk_support_consistency} formalizes this readout step: an estimated column within $\Delta_j / 2$ of the population column $A^*_{:,j}$ in $\ell_\infty$ exactly recovers $\mathcal{S}_j$ via top-$k_j$ selection. The entropy penalty is the mechanism that drives the estimated column close enough to its population counterpart for the influence-gap condition to bite.

\textbf{Theoretical status.}
A tight finite-sample support recovery guarantee for entropy-regularized sparse additive models, analogous to Theorem~\ref{thm:finite} for group-lasso, is not currently available. The obstruction is geometric: $H$ is concave in $p$ on the simplex, but the composition $A \mapsto p(A) \mapsto H(p(A))$ is neither convex nor concave in the unnormalized parameter $A$, so the convex-analysis tools that drive the group-lasso proof do not transfer. Lemma~\ref{lem:entropy_concentration} and Proposition~\ref{prop:topk_support_consistency} together give a partial conditional statement: if column estimation is accurate enough to satisfy the influence gap, support recovery follows. A fully unconditional finite-sample statement remains open. We treat MOSAIC's synthetic evaluations (Appendices~\ref{app:synthetic:nd} and~\ref{app:synthetic:interactions}) as empirical validation of the same qualitative trends predicted by Theorem~\ref{thm:finite}: higher interaction residuals and lower sample-to-dimension ratios both make support recovery harder.

\textbf{Summary.}
Entropy is not a surrogate for group-lasso. The two penalties target different objects: group-lasso regulates column magnitude and is indifferent to internal shape, while entropy regulates column shape and is indifferent to magnitude. The shape-magnitude decoupling is what allows entropy to coexist with reconstruction; the alive mask supplies the on/off decision that scale-invariance prevents entropy from making; and entropy's role in the overall pipeline is finite-sample readout of the population-identifiable supports that Proposition~\ref{prop:no_fp} guarantees.

\subsection{Relation to Direct Jacobian Sparsity}
\label{app:jacobian_sparsity}

An alternative to MOSAIC's additive decoder is to keep a dense decoder $g_\theta(\vz)$ and directly regularize its Jacobian $\partial g_\theta / \partial \vz$ toward sparsity, e.g., via $\ell_1$ on the averaged absolute Jacobian entries $\frac{1}{N}\sum_t |(\partial \hat{\vx}_t / \partial \hat{\vz}_t)_{ij}|$. This approach has two limitations that the additive architecture avoids. First, the Jacobian of a dense MLP is a local quantity: $\partial g_\theta(\vz) / \partial z_j$ depends on the evaluation point $\vz$, so a single factor $z_j$ can produce a non-zero Jacobian entry at some $\vz$ values and a near-zero entry at others. Aggregating over the dataset (e.g., by averaging) loses this local structure and may mask or fabricate support entries. The additive decoder's influence $A_{ij} = |f_j(+1)_i - f_j(-1)_i|$ is a global functional property of $f_j$, independent of where in the latent space it is evaluated. Second, Jacobian sparsity on a dense decoder does not structurally prevent cross-factor entanglement in the decoder: even if $(\partial g / \partial z_j)_i \approx 0$ at most points, the decoder can still route information from $z_j$ to $x_i$ through higher-order paths that the first-order Jacobian does not capture. The additive architecture $\hat{\vx} = \sum_j f_j(z_j) + \mathbf{b}$ eliminates cross-factor interaction by construction, so sparsity in $A$ directly corresponds to the ANOVA main-effect support of Proposition~\ref{prop:no_fp}. The trade-off is representational capacity: the additive decoder cannot fit interaction terms in the mixing function, which is why MOSAIC targets main effects rather than full supports (Section~\ref{sec:problem}).

\section{Robustness and Sensitivity Studies}

\subsection{Ablation and Robustness Studies}
\label{app:ablation}

We ablate the three core components on the synthetic benchmark (5 seeds, $z_{\dim}=12$). Stage~2 ablations share the Stage~1 encoder, so their MCC is identical by construction.

\begin{table}[h]
  \caption{Ablation (5 seeds, $z_{\dim}=12$). $X_Z$@top3 uses the $0.50$ concentration gate.}
  \label{tab:ablation}
  \centering
  \small
  \begin{tabular}{lccc}
    \toprule
    Config & MCC $\uparrow$ & $X_Z$@top3 (gate 0.50) $\uparrow$ & Top-3 mass $\uparrow$ \\
    \midrule
    \textbf{Full MOSAIC}                    & $\mathbf{0.912 \pm 0.009}$ & $\mathbf{0.978 \pm 0.044}$ & $0.685 \pm 0.028$ \\
    w/o sparsity ($\lambda=0$)              & $0.912 \pm 0.009$          & $0.378 \pm 0.259$          & $0.472 \pm 0.031$ \\
    w/o temporal ($\gamma=0$)               & $0.765 \pm 0.020$          & $0.400 \pm 0.181$          & $0.771 \pm 0.062$ \\
    w/o additive (dense)                    & $0.912 \pm 0.009$          & $0.000 \pm 0.000$          & $0.391 \pm 0.008$ \\
    \bottomrule
  \end{tabular}
\end{table}

The two capabilities decompose cleanly: latent identifiability (MCC) depends only on the temporal prior, while support recovery ($X_Z$) requires all three components together. The w/o temporal row is informative: its top-3 mass ($0.771$) exceeds full MOSAIC's ($0.685$), yet $X_Z$ is only $0.400$. Sparse columns alone are insufficient; they must align with correctly identified latents, the same failure mode TDRL exhibits on RNA (Section~\ref{sec:rna}).

\textbf{Robustness checks.} We further verify MOSAIC's behavior under varying conditions: latent dimensionality $z_{\dim}$ sweep (Appendix~\ref{app:synthetic:zdim}), regime-label noise up to $30\%$ flip rate (Appendix~\ref{app:synthetic:regime_noise}), unsupervised regime discovery via k-means (Appendix~\ref{app:synthetic:auto_regime}), and a case study showing how the regime definition determines the scientific question MOSAIC answers (Appendix~\ref{app:synthetic:regime_def}). MOSAIC degrades gracefully in all four settings.

\subsection{z\_dim Sensitivity}
\label{app:synthetic:zdim}

The synthetic benchmark uses $n_{\text{true}} = 6$ true latent factors. MOSAIC's main-text results use $z\_dim = 12$, chosen to be large enough to host all true factors with margin for dead latents. Here we sweep $z\_dim$ across under-, matched-, and over-capacity regimes to characterize MOSAIC's sensitivity to this hyperparameter.

\begin{table}[h]
\centering
\small
\caption{MOSAIC $z_{\dim}$ sensitivity on the synthetic benchmark ($n_{\text{true}}=6$, 5 seeds). We sweep representative under-, matched-, and over-capacity settings under the canonical Stage~2 configuration ($\lambda=50$, $H=4$, 80 epochs, $\gamma=0.02$). ``Pass'' denotes seeds with MCC $\ge 0.85$ and $X_Z$@top3 (gate 0.5) $\ge 0.85$.}
\label{tab:zdim_sensitivity}
\begin{tabular}{r c c c c}
\toprule
$z_{\dim}$ & MCC & $X_Z$@top3 (gate 0.5) & Mean top-3 mass & Seeds passing \\
\midrule
4  & $0.769 \pm 0.054$ & $0.133 \pm 0.130$ & $0.624 \pm 0.028$ & 0 / 5 \\
6  & $0.900 \pm 0.017$ & $0.889 \pm 0.099$ & $0.692 \pm 0.034$ & 4 / 5 \\
8  & $0.912 \pm 0.010$ & $0.844 \pm 0.151$ & $0.698 \pm 0.043$ & 5 / 5 \\
12 & $\mathbf{0.912 \pm 0.009}$ & $\mathbf{0.978 \pm 0.044}$ & $0.685 \pm 0.028$ & $\mathbf{5 / 5}$ \\
16 & $0.906 \pm 0.010$ & $0.733 \pm 0.133$ & $0.751 \pm 0.031$ & 3 / 5 \\
20 & $0.885 \pm 0.017$ & $0.733 \pm 0.249$ & $0.606 \pm 0.085$ & 3 / 5 \\
\bottomrule
\end{tabular}
\end{table}

Under-capacity ($z_{\dim}=4$) breaks both latent recovery and module localization: MCC drops to $0.769$ and no seed clears the joint threshold, consistent with the fact that four latent dimensions cannot represent the six true factors. At matched-to-moderately-over capacity ($z_{\dim}\in\{6,8,12\}$), MCC is stable at $0.90$--$0.91$ and the joint pass rate is high (4/5, 5/5, and 5/5 respectively). Both $z_{\dim}=8$ and $z_{\dim}=12$ achieve full 5/5 seed passing, but $z_{\dim}=12$ reaches the highest $X_Z$@top3 ($0.978$ vs.\ $0.844$) with lower seed-to-seed variance, which is why we adopt it as the main synthetic default.

For stronger over-capacity ($z_{\dim}\in\{16,20\}$), MCC degrades only mildly, indicating that the temporal representation remains largely stable, but $X_Z$@top3 drops to $0.733$ and the joint pass rate falls to 3/5. We observe a recurring failure mode behind this drop. With excess latent capacity, a subset of the extra factors carry a strong regime-association signal (high Cohen's $d$) but receive an influence column that is nearly uniform across the $D$ observed variables, with per-row entries close to $1/D$ and no readable top-$k$ support. These factors are precisely the kind that the alive mask of Appendix~\ref{app:entropy_vs_grouplasso} is designed to neutralize: their column magnitude has not collapsed, so they are not dropped from the entropy sum, yet they have nothing localized to say. Under the canonical $\lambda_{\max}=50$ schedule the entropy term is not strong enough to either concentrate them onto a coherent support or to push their magnitudes below the alive-mask threshold, so they coexist alongside the genuinely localized factors and inflate seed variance.

In practice this failure mode is benign. The diffuse high-$d$ factors are easy to identify by the same diagnostic that drives MOSAIC's interpretation protocol: their influence column has near-uniform mass and no readable top-$k$ support. We therefore recommend, when operating at $z_{\dim} \gg n_{\text{true}}$, ranking factors jointly by regime-association score \emph{and} top-3 mass, and skipping any factor whose top-3 mass is close to the uniform baseline $3/D$. The first regime-discriminative factor with a concentrated influence column carries the scientific signal; the diffuse high-$d$ factors do not introduce false support claims, they simply cannot be interpreted. The 3/5 pass rate at $z_{\dim}\in\{16,20\}$ reflects strict joint-threshold accounting under fixed $\lambda$ rather than an inability to recover the localized factor. Practitioners who do not know $n_{\text{true}}$ in advance should therefore use a moderately overcomplete latent dimension as the structural default, and apply the concentration filter at interpretation time as a safeguard against the over-capacity regime.

\subsection{Multi-Regime Extension}
\label{app:synthetic:multi_regime}

MOSAIC's interpretation protocol is contrastive. Given a scientifically specified regime contrast, it identifies the latent factor most associated with that contrast and interprets the factor through its sparse observation support. When data contain more than two regimes, the intended use is therefore one-vs-reference or pairwise analysis, depending on the scientific question. For example, a user may ask which module distinguishes a particular fault state from normal operation, or which structural module differs between two molecular regimes.

A joint $K$-way analysis is a distinct problem: it requires specifying what object should be recovered across all regimes simultaneously, such as a single latent per regime, a shared latent across multiple contrasts, or groups of latents with overlapping supports. Without this additional target definition, simultaneously explaining all $\binom{K}{2}$ pairwise differences is not the same task as the binary regime-association problem studied in this paper. Richer regime labels can provide stronger auxiliary variation for latent identifiability, but support localization additionally requires resolving assignment, sparsity, and possible shared-mechanism conflicts across contrasts.

Preliminary $K=8$ separable-regime experiments support this distinction: some one-vs-reference contrasts localize well, while others remain diffuse or are assigned to the wrong module under a single global sparsity coefficient. This suggests that robust joint multi-regime localization requires additional machinery, such as contrast-specific assignment, per-regime adaptive sparsity, latent-group selection, or interaction-aware decoders. We therefore leave joint $K$-way support localization to future work and use MOSAIC in this paper as a contrastive interpretation method for binary or user-specified one-vs-reference scientific questions.

\subsection{Regime Label Noise Robustness}
\label{app:synthetic:regime_noise}

Regime labels in real applications come from heuristics (SPIB assignments, thresholding, expert annotation) and may be noisy. We test MOSAIC's degradation by stratified random flips of the canonical binary synthetic benchmark's regime labels at rates $p \in \{0.05, 0.10, 0.20, 0.30\}$, holding the underlying data (dynamics, observations, ground-truth support) fixed. Training uses the noisy labels; evaluation uses the clean ground-truth regime-association identity and support.

\begin{table}[h]
  \caption{Regime label noise robustness (3 seeds per noise rate, z\_dim=12). Labels are perturbed by stratified random flipping at rate $p$ within each regime class. acc\_noisy is the logistic-probe accuracy measured against the noisy labels (the model saw); acc\_clean is the same probe measured against ground-truth clean labels.}
  \label{tab:regime_noise}
  \centering
  \small
  \begin{tabular}{ccccc}
    \toprule
    Noise rate $p$ & MCC & $X_Z$@top3 (gate 0.50) & acc\_noisy & acc\_clean \\
    \midrule
    0.00 & $0.912 \pm 0.010$ & $0.978 \pm 0.044$ & 0.995 & 0.995 \\
    0.05 & $0.891 \pm 0.012$ & $0.960 \pm 0.040$ & 0.942 & 0.992 \\
    0.10 & $0.870 \pm 0.031$ & $0.952 \pm 0.044$ & 0.896 & 0.992 \\
    0.20 & $0.856 \pm 0.045$ & $0.863 \pm 0.082$ & 0.812 & 0.942 \\
    0.30 & $0.808 \pm 0.062$ & $0.827 \pm 0.113$ & 0.732 & 0.892 \\
    \bottomrule
  \end{tabular}
\end{table}

MOSAIC degrades gracefully under label noise. At moderate noise ($p = 0.10$), MCC drops from $0.912$ to $0.870$ and $X_Z$ from $0.978$ to $0.952$. At heavy noise ($p = 0.30$), MCC drops to $0.808$ (11\% relative) and $X_Z$ to $0.827$ (15\% relative). For the practitioner, the 5--10\% noise regime typical of imperfect real-world regime labeling has minimal impact on MOSAIC's outputs.

\section{Related Applications and Future Directions}
\label{app:future_directions}

\textbf{Scientific sensing and industrial monitoring.} Wearable and industrial sensing provide natural settings for extending module discovery. Recent applications include estimating impact force and acceleration with wireless wearable sensors~\citep{chen2026embedded}, combining vibration and acoustic measurements for micro-motor testing~\citep{chen2026sequential}, and using vibration analysis for fault diagnosis and quality control~\citep{chen2026data}. These tasks connect measured signals to physical quantities or operating conditions. A future application of MOSAIC could examine whether distinct latent factors correspond to compact groups of sensor channels or derived features. This would complement prediction and classification by identifying which measurements carry each factor, with physical calibration or known operating changes providing an external basis for interpreting the recovered modules.

\textbf{Latent states for forecasting.} LatentTSF forecasts in a learned latent state space~\citep{yang2026observations}, while TimeAlign aligns historical and future representations to address their distributional gap~\citep{hu2025bridging}. These approaches motivate testing whether MOSAIC's sparse latent modules can provide useful predictive states while retaining their links to named observations.

\textbf{Robustness under input perturbations.} MVI-Bench studies robustness to misleading visual inputs in large vision-language models~\citep{chen2025mvi}. Its focus on input reliability motivates a related evaluation direction for module discovery: testing how recovered supports change under controlled perturbations of observed time series. For scientific sensing, such perturbations could include channel noise, missing measurements, or calibration shifts. Glocal-IB aligns latent representations of masked and originally observed inputs for time-series imputation~\citep{yang2026glocal}, while CRIB uses consistency-regularized representations to forecast directly from partially observed inputs~\citep{yang2025revisiting}. Comparing module stability across these conditions could help distinguish sensitivity to measurement quality from changes associated with the underlying physical process. ConfDock's atom-specific uncertainty quantification for molecular docking~\citep{hao2026confdock} also motivates studying calibrated uncertainty estimates for recovered module assignments in molecular applications.

\textbf{Time-series reasoning systems.} TSRouter selects a modality and model for each time-series reasoning query~\citep{yu2026tsrouter}, while the survey by \citeauthor{hulandscape} places such reasoning within the broader development of agentic time-series systems. MOSAIC could supply these systems with compact descriptions of latent modules and their associated sensor channels, allowing downstream reasoning to refer back to specific measured variables.

\textbf{Interactive scientific analysis.} A further direction is to integrate MOSAIC as a module-discovery tool in an interactive assistant. GAM's graph-based memory~\citep{wu-etal-2026-gam} could organize prior hypotheses and module comparisons across an analysis session. InterruptBench, which evaluates changing user instructions during web navigation~\citep{zou2026userschangemindevaluating}, motivates testing how such an assistant revises its analysis when a researcher changes the target variables or time interval. Agent Security Bench~\citep{zhang2025agent} offers a starting point for assessing attacks involving tool outputs and agent memory in these workflows.

\section{RNA Molecular Dynamics: Extended Analysis}
\label{app:rna}

\textbf{System description.}
The cUUCGg tetraloop (14 nucleotides: \texttt{GGCACUUCGGUGCC}) is a canonical RNA hairpin. We perform MD simulations of this hairpin at three temperatures: 345K (near melting, partial unfolding), 400K (above melting, frequent transitions), and 500K (rapid unfolding); the full simulation protocol is given below in the ``MD protocol'' paragraph. The trajectory inventory consists of 14 input trajectories: 13 curated replicas plus one auxiliary 500K unfolding trajectory. The 500K data does not dominate the disrupted class --- the disrupted frames come predominantly from the long 400K trajectories containing loop/stem disruption events --- but we do not claim temperature-invariant localization in this manuscript.

\textbf{MD protocol.}
All trajectories were generated in OpenMM under the AMBER14 force field with OL3 corrections for RNA, in TIP3P explicit water with neutralizing counterions in a periodic cubic box with at least 10 \AA{} of solvent padding. The initial structure was the cUUCGg tetraloop hairpin (PDB 2KOC), prepared once via energy minimization followed by NPT equilibration at 300\,K and 1\,bar; the post-equilibration checkpoint, together with the serialized system and topology, was reused as the launch state for every replica so that all inter-replica divergence reflects the integrator's stochastic stream rather than initialization noise. Production runs used the Langevin middle integrator ($\gamma = 1$\,ps$^{-1}$) with hydrogen-mass repartitioning ($\mathrm{hydrogenMass} = 3$\,amu) enabling a 4\,fs timestep, a Monte Carlo barostat at 1\,bar with update frequency 25 steps, particle-mesh Ewald electrostatics, and a 10 \AA{} nonbonded cutoff with HBond constraints. Each replica was assigned an independent integrator random seed and velocities re-drawn from the Maxwell--Boltzmann distribution at the target temperature. Trajectories spanned 100\,ns to 1\,$\mu$s in length and were generated at 345\,K, 400\,K, and 500\,K (14 trajectories total). Two trajectory streams were recorded per replica: an RNA-only DCD written every 10\,ps used for downstream feature extraction, and a full-system DCD written every 100\,ps retained as an archival reference. Compute resources were the Bridges2 cluster (PSC) and Google Colab Pro (A100 GPU).

\textbf{Feature extraction.}
For each of the 14 residues, we compute the mean and standard deviation of its minimum heavy-atom distance to every other residue at each frame. This yields $D = 28$ per-residue distance features (14 mean + 14 std). Distance features directly capture base-pairing module structure: paired residues (e.g., G1--C14) co-vary in distance space, while unpaired loop residues (U6--G9) fluctuate independently.

\textbf{Regime definition.}
Regimes are defined by the closing-pair native-contact fraction $Q_{\text{cp}}$. Candidate residue pairs are sequence-separated ($|i-j| \ge 3$) closest-heavy contacts that touch C5 or G10. Native contacts are selected from frame 0 of each trajectory, not from an external 2KOC reference structure. A pair is included as native if $d_{ab}(0) < 0.45$\,nm ($4.5$\,\AA{}), and it is present at frame $t$ iff $d_{ab}(t) < 1.2\,d_{ab}(0)$. Thus $Q_{\text{cp}}(t)$ is the mean of binary contact indicators over the native C5/G10-touching contact set; contacts are binary, not distance-weighted.

A 51-frame rolling mean is applied before thresholding to suppress thermal noise. Frames are labelled intact when $Q_{\text{cp}}>0.8$ (regime~0) and disrupted when $Q_{\text{cp}}<0.65$ (regime~1); intermediate frames and a $\pm 50$-frame buffer around label transitions are discarded. From stored labels, $24.1\%$ of frames are discarded after smoothing and transition-buffer removal; this combines intermediate-band frames and transition-buffer frames because $Q_{\text{cp}}$ itself was not stored in the diagnostics file. After lag-2 windowing (sequence length 3) and class balancing, the dataset contains $156{,}170$ windows ($78{,}085$ per regime, $N/D \approx 5{,}571$). For cross-regime validation, we also train with an outer-stem contact criterion $Q_{\text{stem}}$ using torsion-angle features ($D = 56$).

\textbf{Scope of RNA experiments.}
We have not run threshold-sensitivity, three-regime, or temperature-specific MOSAIC experiments in this manuscript. The reported RNA result is a pooled-temperature analysis; we do not claim temperature-invariant localization.

\textbf{Ground-truth support structure.}
Based on established RNA structural biology, we define three modules corresponding to the physical regions of the hairpin:

\begin{table}[h]
  \caption{Ground-truth support structure for the RNA tetraloop. Each module corresponds to a structural region of the hairpin.}
  \centering
  \small
  \resizebox{\textwidth}{!}{%
  \begin{tabular}{llcl}
    \toprule
    Module & Residues & Variables & Physical role \\
    \midrule
    Stem & G1, G2, C3, A4, U11, G12, C13, C14 & 16 & Base-paired stem (last to unfold) \\
    ClosingPair & C5, G10 & 4 & Stem-loop interface (commitment gate) \\
    Loop & U6, U7, C8, G9 & 8 & UUCG tetraloop (first to fluctuate) \\
    \bottomrule
  \end{tabular}%
  }
\end{table}

\textbf{Extended results.}
With $n = 4$ latent factors, MOSAIC achieves regime accuracy $0.912 \pm 0.027$ (5 seeds) and mean concentration $0.617 \pm 0.052$. The top regime-associated latent (highest Cohen's $d$, mean $d = 1.48 \pm 0.24$) localizes to Loop in 4 of 5 seeds and Stem in the remaining seed. Across all 5 seeds, the rank ordering of module influence is stable: Loop\,$>$\,Stem\,$>$\,ClosingPair, consistent with the known thermodynamic coupling of stem unpairing and loop melting in RNA hairpin unfolding~\citep{chen2013rna}.

\begin{table}[h]
  \caption{RNA regime-association summary for MOSAIC (5 seeds), moved from the main text for compactness.}
  \label{tab:rna_driver_app}
  \centering
  \small
  \begin{tabular}{lcc}
    \toprule
    Setting & Regime Acc $\uparrow$ & Top regime-associated module \\
    \midrule
    \textbf{MOSAIC (cp\_Q, 5 seeds)} & $\mathbf{0.912 \pm 0.027}$ & Loop (\#1 in 4/5) \\
    \bottomrule
  \end{tabular}
\end{table}

\paragraph{Seed-level diagnostic for the non-Loop-first run.}
The only non-Loop-first MOSAIC seed is not a failure to capture the regime signal: the selected regime-associated latent still has strong regime association (Cohen's $d=1.395$). Its support, however, is routed through stem and stem-junction residues rather than purely through Loop residues. In this seed, five of the top-8 non-Loop entries come from stem positions, consistent with stem--loop co-variation during closing-pair disruption. This suggests a boundary case of the discrete Loop/ClosingPair/Stem annotation rather than an unrelated localization: the transition signal remains physically coupled to the loop-opening process, but the sparse decoder assigns part of the support to correlated stem residues. Under the enhanced gap-adaptive support diagnostic, this seed retains Loop recall $0.75$, indicating that most Loop residues remain present in the learned support even though they do not dominate the top-8 precision metric.

\textbf{Baseline comparison.}
Table~\ref{tab:rna_baseline_app} compares MOSAIC to CRL, linear, and information-bottleneck baselines under the unified $X_Z$@selected-latent metric (precision of the top regime-associated factor's top-8 influence entries against Loop's 8 ground-truth variables). MOSAIC achieves $0.880 \pm 0.055$, exceeding every baseline by at least 0.13. All baselines except PCA still rank Loop as their top regime-associated module by Cohen's $d$, confirming that the Loop signal is strong, but only MOSAIC concentrates its influence cleanly on the ground-truth Loop support. PCA is a useful failure case: it ranks Stem first rather than Loop, consistent with variance-dominant behavior in unregularized linear decomposition, where high-variance structural coordinates dominate principal components even when they are not the main regime-shift factor. SlowVAE achieves the highest regime accuracy ($0.955$) but the lowest $X_Z$@selected-latent among CRL methods ($0.562$), illustrating the distinction between regime separability and module localization.

\begin{table}[h]
  \caption{RNA baseline comparison using the canonical $X_Z$@selected-latent metric (precision of the top regime-associated factor's top-8 influence entries against Loop's 8 variables). MOSAIC: 5-seed mean $\pm$ std; CRL baselines: 4 seeds (42/123/456/789); linear baselines and SPIB: single seed or 5-seed mean as noted.}
  \label{tab:rna_baseline_app}
  \centering
  \small
  \begin{tabular}{lcc}
    \toprule
    Method & Regime Acc $\uparrow$ & $X_Z$@selected (k=8) $\uparrow$ \\
    \midrule
    \textbf{MOSAIC (ours, 5s)} & $\mathbf{0.912 \pm 0.027}$ & $\mathbf{0.880 \pm 0.055}$ \\
    \midrule
    TDRL (4s) & $0.912 \pm 0.030$ & $0.625 \pm 0.102$ \\
    iVAE (4s) & $0.930 \pm 0.003$ & $0.594 \pm 0.157$ \\
    SlowVAE (4s) & $0.955 \pm 0.001$ & $0.562 \pm 0.217$ \\
    CtrlNS (4s) & $0.929 \pm 0.011$ & $0.562 \pm 0.161$ \\
    $\beta$-VAE (4s) & $0.858 \pm 0.028$ & $0.469 \pm 0.120$ \\
    \midrule
    Sparse PCA ($\alpha{=}10$) & 0.898 & 0.750 \\
    Sparse PCA ($\alpha{=}1$) & 0.902 & 0.625 \\
    FastICA & 0.901 & 0.625 \\
    PCA & 0.901 & 0.625 \\
    \midrule
    SPIB (5s) & $0.645 \pm 0.065$ & $0.400 \pm 0.137$ \\
    \bottomrule
  \end{tabular}
\end{table}

For cross-regime validation, retraining with $Q_{\text{stem}}$ labels shifts the dominant module from Loop to Stem (Table~\ref{tab:rna_cross_regime}). This supports a non-circular interpretation: the selected regime-associated latent changes its localized support with the physical regime definition, rather than always pointing to a fixed residue group.

\textbf{Regime design and scientific target.}
The cross-regime result also clarifies a practical use pattern: the regime definition determines the scientific question that the regime-association ranking answers. If the goal is to compare two steady states (State A vs.\ State B), regimes should be defined directly by those states. If the goal is to study transition dynamics (how State A moves toward State B), regimes can be defined as stable State A versus pre-transition State A. Under either choice, MOSAIC returns the factor with the strongest cross-regime shift under that task-specific definition.

\begin{table}[h]
  \caption{Cross-regime validation on RNA: the top regime-associated module shifts with the physical regime definition, tracking the structural layer adjacent to the contact set used for regime labeling.}
  \label{tab:rna_cross_regime}
  \centering
  \small
  \begin{tabular}{lll}
    \toprule
    Regime definition & Feature space & Localized module of selected latent \\
    \midrule
    Closing-pair contacts ($Q_{\text{cp}}$) & Per-residue distances ($D{=}28$) & Loop \\
    Outer-stem contacts ($Q_{\text{stem}}$) & Torsion angles ($D{=}56$) & Stem \\
    \bottomrule
  \end{tabular}
\end{table}

\section{Cross-Domain Details}
\label{app:crossdomain}

Each cross-domain dataset is preprocessed per Appendix~\ref{app:hyperparams}. Regime definitions are as follows. OMNI~\citep{king2005omni} uses a geomagnetic-activity threshold on the $Kp$ index. The tokamak disruption benchmark is synthetic, generated by an AR(1) process ($\rho = 0.95$, seed=42) over four latent sources (MHD instability, Density, Energy confinement, Plasma shape), mapped to 12 observables through a sparse mixing matrix; the pre-disruptive regime (regime 1) activates the MHD and Density sources with a ramp in the second half of each shot. Construction details are in Appendix~\ref{app:crossdomain:disruption}, and variable selection follows tokamak physics~\citep{devries2011survey}. Climate uses ENSO phase derived from NOAA ERSSTv5 sea surface temperature data~\citep{huang2017extended,huang2017ersstv5data}. TEP~\citep{downs1993plant} uses a binary fault-vs-normal label where fault instances are drawn from IDV-1, IDV-2, IDV-4, and IDV-5 (feed-ratio, B-composition, reactor-cooling, and condenser-cooling faults respectively). This multi-fault regime design tests whether MOSAIC can recover a shared downstream signature when multiple distinct upstream disturbances are labeled jointly.

Throughout these diagnostics, the selected latent is the one with largest distributional discrepancy across the specified regimes. This association-based selection should not be interpreted as an interventional causal effect of the latent on the transition.

\begin{table}[h]
  \caption{Additional numerical details for the cross-domain benchmarks (MOSAIC, 3 seeds per dataset). ``Top $d$'' reports the observed range of the leading Cohen's $d$ across seeds, and ``Expected regime-associated variables'' summarizes the variables on which the dominant factor loads.}
  \label{tab:crossdomain_details}
  \centering
  \small
  \resizebox{\textwidth}{!}{%
  \begin{tabular}{llccc}
    \toprule
    Dataset & Expected regime-associated variables & Regime Acc & Top $d$ & Consistent \\
    \midrule
    OMNI & Coupling-related solar-wind parameters & 0.87--0.98 & 1.54--1.98 & 3/3 \\
    Disruption & MHD/Density group ($q_{95}$, greenwald\_frac) & 0.91 & 1.96--2.08 & 3/3 \\
    Climate & Ni\~no 3.4 and tropical Pacific indices & 0.85--0.92 & 1.44--1.83 & 3/3 \\
    TEP & Stripper group (product composition XMEAS\_37--41) & $0.729 \pm 0.065$ & $0.743$ (range $0.554$--$0.862$) & 3/3 \\
    \bottomrule
  \end{tabular}%
  }
\end{table}

\begin{table}[h]
  \caption{Signal-strength separation on the synthetic tokamak benchmark. Both $q_{95}$ and locked-mode amplitude are generated from the same MHD latent source, so both receive the pre-disruptive ramp scaled by their respective mixing coefficients. However, locked-mode's smaller mixing coefficient produces a much weaker ramp signal (Cohen's $d = 0.24$ vs.\ $2.08$), causing MOSAIC to assign it to a separate, weakly regime-associated factor rather than grouping it with $q_{95}$. This is a limitation: MOSAIC's support recovery is signal-strength dependent, and weak-signal channels from the same source may be split into separate factors rather than recovered as a single module.}
  \label{tab:disruption_locked_mode_app}
  \centering
  \small
  \begin{tabular}{lccl}
    \toprule
    Latent factor & Cohen's $d$ & Top loaded variable & Role \\
    \midrule
    $z_3$ & 2.08 & $q_{95}$ & Regime-varying (strong signal) \\
    $z_0$ & 0.24 & Locked-mode amp. & Weakly regime-associated (weak signal) \\
    \bottomrule
  \end{tabular}
\end{table}

Across the four cross-domain benchmarks where the additive-sparse assumption is met (OMNI, Disruption, Climate, TEP), the dominant factor loads on the domain-expected variables in all three seeds per dataset, for 12 of 12 seeds in aggregate.

\textbf{Linear baseline comparison.}
Table~\ref{tab:crossdomain_linear} reports the available deterministic linear baselines on three cross-domain benchmarks (OMNI, Disruption, Climate). SparsePCA ($\alpha{=}1$) achieves domain-consistent regime-associated localization on all three reported datasets, matching MOSAIC on this coarse domain-knowledge criterion for the same datasets. This means the cross-domain result should be read as transfer evidence for MOSAIC's interpretation protocol rather than as a universal dominance claim; MOSAIC's added value is identifiable nonlinear temporal latents and cleaner separation of regime-associated from variance-dominant factors. PCA fails on OMNI and Climate, each identifying variance-dominant rather than regime-relevant variables.

On Disruption, all linear methods include locked-mode amplitude in their top-3 regime-associated variables alongside $q_{95}$. MOSAIC instead assigns locked-mode to a separate, weakly regime-associated factor (Table~\ref{tab:disruption_locked_mode_app}). Since both observables are generated from the same MHD source, including locked-mode alongside $q_{95}$ (as linear methods do) is in fact consistent with the ground-truth module structure. MOSAIC's separation reflects signal-strength thresholding rather than true module discovery in this case: the weak ramp signal on locked-mode falls below the concentration threshold, splitting a single ground-truth module across two learned factors.

\begin{table}[h]
  \caption{Available linear baselines on three cross-domain benchmarks (single seed, deterministic). ``Consistent'' uses the same domain-knowledge criteria as Table~\ref{tab:cross_domain}.}
  \label{tab:crossdomain_linear}
  \centering
  \small
  \begin{tabular}{llcc}
    \toprule
    Dataset & Method & Regime Acc & Consistent? \\
    \midrule
    OMNI & SparsePCA $\alpha{=}1$ & 0.992 & YES \\
    OMNI & SparsePCA $\alpha{=}10$ & 0.992 & YES \\
    OMNI & FastICA & 0.992 & YES \\
    OMNI & PCA & 0.992 & NO \\
    \midrule
    Disruption & SparsePCA $\alpha{=}1$ & 0.924 & YES \\
    Disruption & SparsePCA $\alpha{=}10$ & 0.924 & YES \\
    Disruption & FastICA & 0.925 & YES \\
    Disruption & PCA & 0.924 & YES \\
    \midrule
    Climate & SparsePCA $\alpha{=}1$ & 0.956 & YES \\
    Climate & SparsePCA $\alpha{=}10$ & 0.975 & YES \\
    Climate & FastICA & 0.936 & YES \\
    Climate & PCA & 0.934 & NO \\
    \bottomrule
  \end{tabular}
\end{table}

\textbf{SPIB on cross-domain.}
SPIB achieves domain-consistent regime-associated localization in 5 of 9 reported runs (Table~\ref{tab:crossdomain_spib}), compared to 9/9 for MOSAIC on the same three datasets (OMNI, Disruption, Climate). On OMNI, SPIB fails all three seeds; its early-stopping convergence criterion triggers within 2 to 4 epochs.

\begin{table}[h]
  \caption{SPIB on cross-domain benchmarks (3 seeds each).}
  \label{tab:crossdomain_spib}
  \centering
  \small
  \begin{tabular}{lcc}
    \toprule
    Dataset & Regime Acc & Consistent? \\
    \midrule
    OMNI & $0.696 \pm 0.105$ & 0/3 \\
    Disruption & $0.763 \pm 0.053$ & 3/3 \\
    Climate & $0.750 \pm 0.066$ & 2/3 \\
    \bottomrule
  \end{tabular}
\end{table}

\subsection{Synthetic Tokamak Disruption Benchmark: Construction Details}
\label{app:crossdomain:disruption}

The tokamak disruption benchmark is a physics-motivated synthetic dataset. We use synthetic construction because accessing per-shot time-series diagnostic data from operational tokamaks (JET, DIII-D, ASDEX Upgrade) requires facility-specific data access agreements not within the scope of this work; the ITPA disruption database contains aggregate shot-level summary statistics rather than the per-shot time series required for temporal representation learning. We instead construct a benchmark whose variable names and module structure follow established tokamak physics to preserve the scientific interpretation of group-mean winners.

\textbf{Generation.} Four latent sources (MHD instability, Density, Energy confinement, Plasma shape) evolve as AR(1) processes with autocorrelation $\rho = 0.95$ and innovation standard deviation $0.3$, seed $= 42$. A sparse block-structured mixing matrix $M \in \real^{4 \times 12}$ maps sources to observables: MHD source maps to $q_{95}$, $l_i$, locked\_mode; Density source maps to $\bar{n}$, greenwald\_frac; Energy source maps to $w_{\mathrm{MHD}}$, $\beta_N$, $p_{\mathrm{rad}}/p_{\mathrm{in}}$; Shape source maps to $I_p$, elongation, triangularity, $\delta$. Mixing coefficients are drawn uniformly from $[0.6, 1.5]$. Gaussian observation noise ($\sigma = 0.15$) is added.

\textbf{Regime labels.} In regime 0 (healthy), sources evolve under the AR(1) dynamics alone. In regime 1 (pre-disruptive), the MHD and Density sources receive an additive ramp in the second half of each shot: $s_t^{(\mathrm{MHD})} \mathrel{+}= 1.5 \cdot \mathrm{progress}_t$ and $s_t^{(\mathrm{Density})} \mathrel{+}= 0.8 \cdot \mathrm{progress}_t$, where $\mathrm{progress}_t = (t - T/2) / T$ for $t > T/2$ and zero otherwise. This reflects the physics-literature pattern in which MHD mode growth and density-limit approach precede disruption onset~\citep{devries2011survey}. MOSAIC's recovery of MHD or Density as the regime-associated group is consistent with this ground-truth construction.

\textbf{Sample size.} The processed synthetic tokamak disruption benchmark contains $N=4{,}600$ lagged windows with sequence length 3 and $D=12$ variables, balanced across healthy and pre-disruptive regimes ($2{,}300$ windows each). It was generated with $100$ trajectories per class, shot length $50$, and lag $2$; the reported $N/D=383$ is the sample-to-dimension ratio, not the number of samples per regime.

\textbf{Scope.} The variable names, module grouping, and pre-disruptive ramp dynamics are drawn from published tokamak physics. The sparse mixing structure matches MOSAIC's additive-sparse assumption more cleanly than raw tokamak diagnostics might; this benchmark therefore provides a controlled evaluation of MOSAIC on tokamak-relevant physical structure, not a substitute for deployment on real disruption prediction.

\subsection{TEP: Multi-Fault Regime}
\label{app:crossdomain:tep}

The Tennessee Eastman Process (TEP) is a well-studied chemical-process benchmark with 52 observed variables (41 measurements \texttt{XMEAS\_1}--\texttt{XMEAS\_41} and 11 manipulated variables \texttt{XMV\_1}--\texttt{XMV\_11}) grouped into Reactor, Condenser, Separator, Stripper, Compressor, and Feed/Global physical clusters. The normal-vs.-faulty regime split uses $N=4{,}176$ post-windowing samples ($D=52$, $N/D \approx 80$). We label regime 0 as normal operation (IDV-0) and regime 1 as jointly drawn from IDV-1 (stream-A feed ratio fault), IDV-2 (stream-B composition fault), IDV-4 (reactor cooling fault), and IDV-5 (condenser cooling fault). Under this multi-fault labeling, MOSAIC's top-1 influence variable on the most regime-associated latent is \texttt{XMEAS\_1} (stream-A feed flow, the IDV-1 primary fault variable) across all 3 seeds (seeds 0, 42, 123; selected latents $z_4$, $z_5$, $z_3$ respectively), and the per-group mean winner is the Stripper group (containing Stream-11 product composition variables \texttt{XMEAS\_37}--\texttt{XMEAS\_41}, which are the downstream response to all four fault types). Regime accuracy is $0.729 \pm 0.065$ and the top Cohen's $d$ ranges from $0.554$ to $0.862$ (mean $0.743$). The winner-vs-runner-up group-mean ratio is $1.07$, substantially smaller than Disruption's $1.38$--$2.13$: this reflects the multi-fault regime structure, in which no single fault's primary variable group dominates because the regime-associated latent must encode four distinct fault types jointly, and the Stream-11 product-composition cluster is the common downstream observable of all four faults. Under a single-fault regime (e.g., IDV-1 only), we would expect a sharper group-mean separation; the multi-fault benchmark represents a harder setting.

\section{Interaction Ratio Calibration on Real Data}
\label{app:rho_calibration}

Theorem~\ref{thm:finite} predicts that support recovery becomes more sample-demanding as the interaction ratio $\rho$ grows. To locate RNA and the four cross-domain benchmarks on this empirical scale, we estimate $\hat\rho$ using only quantities already computed during MOSAIC training.

\textbf{Estimator.}
Let $\hat{\mathbf{z}}_t$ denote the Stage~1 encoder posterior mean and $\hat g^{\mathrm{add}}(\hat{\mathbf{z}}_t) = \sum_j \hat f_j(\hat z_{t,j}) + \hat{\mathbf{b}}$ the Stage~2 additive-decoder reconstruction. Write $\hat{\sigma}_{\epsilon,\mathrm{ub}}^2 = \frac{1}{N} \sum_t \|\vx_t - \hat g^{\mathrm{dense}}(\hat{\mathbf{z}}_t)\|_2^2$ for the Stage~1 dense-decoder reconstruction MSE, which serves as an upper-bound proxy for the true observation-noise variance $\sigma_\epsilon^2$. The empirical interaction-residual variance is
\begin{equation}
\hat{\sigma}_r^2 \;=\; \frac{1}{N} \sum_{t=1}^{N} \bigl\| \vx_t - \hat g^{\mathrm{add}}(\hat{\mathbf{z}}_t) \bigr\|_2^2 \;-\; \hat{\sigma}_{\epsilon,\mathrm{ub}}^2,
\end{equation}
obtained by subtracting the noise-floor proxy from the additive-decoder residual. Writing $\hat{\sigma}_g^2 = \frac{1}{N}\sum_t \|\vx_t - \bar{\vx}\|_2^2 - \hat{\sigma}_{\epsilon,\mathrm{ub}}^2$ for the total signal variance (with $\bar{\vx} = \frac{1}{N}\sum_t \vx_t$), the interaction ratio is estimated by $\hat\rho = \hat{\sigma}_r^2 / \hat{\sigma}_g^2$. Because $\hat{\sigma}_{\epsilon,\mathrm{ub}}^2$ is an upper bound on true noise, subtracting it reduces both numerator and denominator. Since the numerator (interaction variance) is smaller than the denominator (total signal variance), the fractional reduction is larger in the numerator, so $\hat\rho$ is a lower bound on the true $\rho$. This makes the estimator optimistic about additivity: if anything, the true interaction contribution is larger than the reported $\hat\rho$, which means MOSAIC's empirical support recovery on these datasets is achieved under interaction levels at least as strong as those reported in Table~\ref{tab:rho_calibration}.

\textbf{Results.}
Table~\ref{tab:rho_calibration} reports $\hat\rho$ on the five evaluated domains, together with the input dimensionality $D$, sample-to-dimension ratio $N/D$, and the primary recovery metric for each dataset. For RNA, the recovery metric is $X_Z$@selected-latent from Appendix~\ref{app:rna}; for the four cross-domain benchmarks, ground-truth support is not available, so we report regime accuracy from Table~\ref{tab:cross_domain} where available.

\begin{table}[h]
  \caption{Estimated interaction ratio $\hat\rho$ on the five evaluated domains. Smaller $\hat\rho$ means the mixing is closer to exactly additive, and main-effect recovery is expected to be more accurate. TEP is a 3-seed estimate (mean $\pm$ std); other estimable domains are single-ckpt estimates. $^\dagger$Climate ($N/D = 13$) is marked ``---'' because its $N/D$ ratio is too low for the nonparametric additive-projection estimator to yield a reliable $\hat\rho$; MOSAIC's regime accuracy on this dataset ($0.89$) confirms it is not in a catastrophic-interaction regime.}
  \label{tab:rho_calibration}
  \centering
  \small
  \begin{tabular}{lrrcc}
    \toprule
    Dataset & $D$ & $N/D$ & $\hat\rho$ & Recovery metric \\
    \midrule
    RNA        & 28 & 5{,}571 & $0.161$ & $0.880 \pm 0.055$ ($X_Z$@selected-latent) \\
    OMNI       & 21 & 1{,}277 & $0.045$ & $0.940 \pm 0.054$ (Regime Acc) \\
    TEP        & 52 & 80      & $\mathbf{0.018 \pm 0.006}$ & $0.729 \pm 0.065$ (Regime Acc) \\
    Disruption & 12 & 383     & $0.003$ & $0.911 \pm 0.003$ (Regime Acc) \\
    Climate    & 47 & 13      & ---$^\dagger$ & $0.888 \pm 0.027$ (Regime Acc) \\
    \bottomrule
  \end{tabular}
\end{table}

\textbf{Calibration anchor.} On the synthetic benchmark, where the ground-truth mixing is additive ($\rho=0$ by construction), the same estimator returns $\hat\rho_{\text{synth}} = 0.024$ (median over knot/CV grid), establishing a finite-sample calibration floor. All real-domain estimates above sit at or modestly above this floor, consistent with the additive-mixing assumption being empirically defensible.

\textbf{Comparison with the synthetic sweep.}
On the controlled interaction sweep in Appendix~\ref{app:synthetic:interactions}, three of the four estimable evaluated domains (OMNI at $0.045$, TEP at $0.018$, Disruption at $0.003$) fall inside the calibrated range with $\hat\rho \leq 0.05$. RNA's $\hat\rho = 0.161$ lies beyond it, reflecting allosteric coupling richer than the synthetic benchmark's pairwise $\tanh$ construction can produce; MOSAIC still localizes the selected latent to Loop in 4 of 5 seeds, showing that the empirical method can remain useful outside the lowest-interaction calibrated range.

Figure~\ref{fig:rho_calibration} summarizes the joint behavior of $\hat\rho$, synthetic interaction strength $\alpha$, and MOSAIC's support-recovery and regime-association metrics, with the four estimable evaluated domains marked on the empirical degradation curve.

\begin{figure}[h]
\centering
\includegraphics[width=\textwidth]{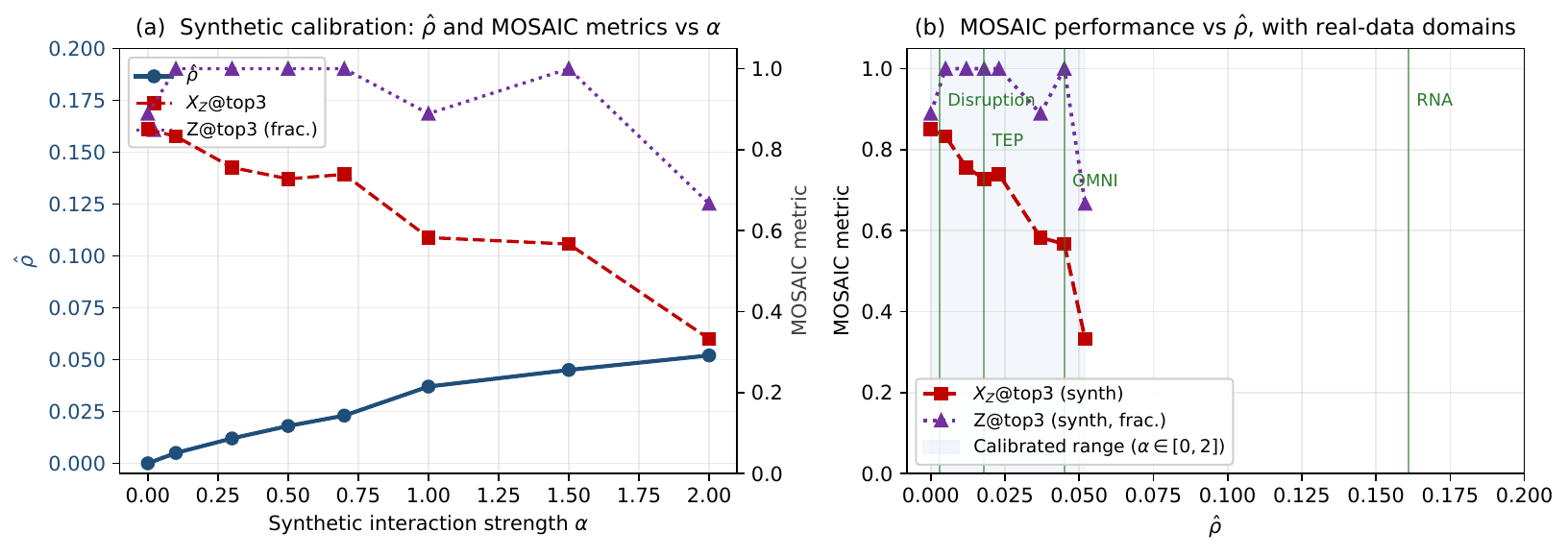}
\caption{Interaction-ratio calibration. (a) On the synthetic interaction sweep, $\hat{\rho}$ (left axis) grows with the interaction strength $\alpha$, while MOSAIC metrics (right axis) show that the stricter module-localization score $X_Z$@top3 degrades earlier than the latent-level regime-association score $Z$@top3. (b) Plotting MOSAIC metrics against $\hat{\rho}$ places the estimable evaluated domains on the empirical degradation curve: OMNI, TEP, and Disruption fall inside the synthetic calibrated range (shaded), while RNA ($\hat{\rho}=0.161$) lies beyond it but still yields the expected Loop-associated localization. Climate is omitted because its $N/D$ ratio is insufficient for the $\hat{\rho}$ estimator used here (Table~\ref{tab:rho_calibration}).}
\label{fig:rho_calibration}
\end{figure}

\section{Concentration Metric Details}
\label{app:concentration}

The concentration metric $\text{conc}(i) = \max_j A_{ij} / \sum_j A_{ij}$ measures how strongly each variable is dominated by a single latent factor. At the uniform baseline ($\nicefrac{1}{n}$), every factor contributes equally and no module structure is present. At concentration 1.0, the variable is perfectly assigned to one factor. On RNA, the mean concentration is $0.617 \pm 0.052$ (5 seeds), indicating that most variables are predominantly governed by a single latent factor. Consistent with Theorem~\ref{thm:finite}, concentration correlates with $N/D$ across domains: datasets with higher $N/D$ ratios achieve sharper support assignments.

\section{Concentration Gate Sensitivity}
\label{app:concentration_sensitivity}

$X_Z$@top3 in Table~\ref{tab:synthetic} uses a concentration gate at top-3 mass threshold $0.50$: a regime-discriminative latent contributes to the score only if its top-3 influence mass accounts for at least half of the total column mass. Ungated $X_Z$@top3 measures top-rank overlap with the ground-truth support, whereas the concentration gate tests whether a factor yields a localized module-level explanation; a dense decoder with partial top-rank overlap can still fail the gate because its influence is spread across many observed variables. TDRL illustrates this pattern: its ungated score is $0.519$, so there is nonzero overlap with the true support, but its top-3 mass averages only $\approx 25\%$ and the gate correctly reports diffuse mass rather than a zero-overlap failure. Table~\ref{tab:concentration_sensitivity} reports $X_Z$@top3 across a range of thresholds. The gate at $0.50$ is the permissive choice: MOSAIC's mean top-3 mass is $0.70$, so it satisfies the gate on all seeds. Gates at $0.60$ and $0.70$ penalize MOSAIC's own moderate seeds as well, reducing discriminative power. Without any gate, iVAE's dense decoder scores $0.863$ purely from argmax alignment, masking the absence of actual support concentration.

\begin{table}[h]
  \caption{$X_Z$@top3 sensitivity to the concentration-gate threshold $\mu_0$ on the synthetic benchmark. Gate at $0.50$ (used in Table~\ref{tab:synthetic}) cleanly separates sparse-decoder methods (MOSAIC, iVAE) from dense-decoder methods (TDRL, SlowVAE, $\beta$-VAE). MOSAIC is unaffected at $0.50$ ($X_Z = 0.978$), moderately penalized at $0.60$ ($X_Z = 0.844$), and more heavily penalized at $0.70$ ($X_Z = 0.511$) --- indicating its top-3 captures approximately $65\%$--$70\%$ of column mass on average. iVAE's decoder, despite argmax correctness under no gate, places only $\approx 50\%$ of mass in the top-3, causing sharp degradation.}
  \label{tab:concentration_sensitivity}
  \centering
  \small
  \begin{tabular}{lcccc}
    \toprule
    Method & gate $0.50$ & gate $0.60$ & gate $0.70$ & no gate \\
    \midrule
    MOSAIC   & $0.978$ & $0.844$ & $0.511$ & $0.978$ \\
    iVAE     & $0.556$ & $0.222$ & $0.000$ & $0.863$ \\
    TDRL     & $0.000$ & $0.000$ & $0.000$ & $0.519$ \\
    SlowVAE  & $0.000$ & $0.000$ & $0.000$ & $0.407$ \\
    $\beta$-VAE & $0.000$ & $0.000$ & $0.000$ & $0.407$ \\
    \bottomrule
  \end{tabular}
\end{table}

\section{SPIB Hyperparameter Sweep on the Synthetic Benchmark}
\label{app:spib_sweep}

SPIB underperforms MOSAIC and most other baselines on the synthetic benchmark (Table~\ref{tab:synthetic}), with regime accuracy $0.304 \pm 0.019$ and zero gated $X_Z$@top3. We swept SPIB hyperparameters to verify that this is a structural mismatch with the support-recovery objective rather than a poorly tuned baseline.

\textbf{Sweep design.} We swept $\beta \in \{0.001, 0.01, 0.1, 1.0\}$ and $K \in \{2, 4, 8\}$, the two hyperparameters that govern SPIB's information-bottleneck strength and the number of metastable states. For $K > 2$, we enabled SPIB's iterative label refinement (\texttt{refinements=5}) so that extra output slots could be populated from the binary ground-truth regime labels. All 12 configurations were trained on the synthetic benchmark at seed 42.

\begin{table}[h]
\centering
\small
\caption{SPIB hyperparameter sweep on the synthetic benchmark, seed 42. Regime accuracy is the logistic-probe accuracy on $\hat{\mathbf{z}}$; MCC, Z@top3, and $X_Z$@top3 are defined identically to Table~\ref{tab:synthetic}. ``Driver $\hat{z}^*$'' is the learned latent index with largest Cohen's $d$ under each configuration; these are learned latent indices, not ground-truth factor labels.}
\label{tab:spib_sweep}
\begin{tabular}{r r c c c c l c}
\toprule
$\beta$ & $K$ & Regime acc & MCC & Z@top3 & $X_Z$@top3 & Driver $z^*$ & Cohen's $d$ \\
\midrule
0.001 & 2 & 0.705 & 0.274 & 0/3 & 0.000 & $z_0$ & 0.538 \\
0.001 & 4 & \textbf{0.814} & \textbf{0.464} & 0/3 & 0.000 & $z_2$ & 1.361 \\
0.001 & 8 & 0.791 & 0.267 & 0/3 & 0.000 & $z_5$ & 1.030 \\
0.01  & 2 & 0.724 & 0.320 & 0/3 & 0.000 & $z_3$ & 0.788 \\
0.01  & 4 & 0.676 & 0.191 & 0/3 & 0.000 & $z_6$ & 0.647 \\
0.01  & 8 & 0.667 & 0.208 & 0/3 & 0.000 & $z_1$ & 0.485 \\
0.1   & 2 & 0.698 & 0.287 & 0/3 & 0.000 & $z_3$ & 0.602 \\
0.1   & 4 & 0.717 & 0.303 & 0/3 & 0.000 & $z_0$ & 0.776 \\
0.1   & 8 & 0.630 & 0.222 & 0/3 & 0.000 & $z_5$ & 0.477 \\
1.0   & 2 & 0.711 & 0.306 & 0/3 & 0.000 & $z_3$ & 0.683 \\
1.0   & 4 & 0.701 & 0.271 & 0/3 & 0.000 & $z_3$ & 0.644 \\
1.0   & 8 & 0.699 & 0.233 & 0/3 & 0.000 & $z_7$ & 0.619 \\
\bottomrule
\end{tabular}
\end{table}

\textbf{Two patterns emerge.} First, regime accuracy is moderately tunable: 7 of 12 configurations exceed $0.70$, and the best ($\beta=0.001$, $K=4$) reaches $0.814$. The Table~\ref{tab:synthetic} number of $0.304$ used SPIB's default hyperparameters, which fall outside this tunable range. Second, and more importantly, identification metrics are uniformly zero: every one of the 12 configurations achieves Z@top3 $=0/3$ and $X_Z$@top3 $=0.000$. Even the strongest-tuned configuration in regime accuracy, with Cohen's $d$ on its top latent reaching $1.361$, fails to place any of its top-3 latents on the ground-truth regime-varying factors $\{z_0, z_1, z_2\}$.

\textbf{Seed stability of the best configuration.} We re-ran the best configuration ($\beta=0.001$, $K=4$) at two additional seeds. Regime accuracy is reproducible at $0.843 \pm 0.028$ (seeds 0/42/123), confirming that SPIB is genuinely tunable on regime prediction. MCC is more variable ($0.315 \pm 0.113$), but Z@top3 and $X_Z$@top3 remain identically zero across all three seeds, with different ``winner'' latents ($z_7$, $z_2$, $z_0$) selected as most regime-discriminative in each run.

\begin{table}[h]
\centering
\small
\caption{Seed stability of the best SPIB configuration ($\beta=0.001$, $K=4$).}
\label{tab:spib_seeds}
\begin{tabular}{r c c c c l c}
\toprule
Seed & Regime acc & MCC & Z@top3 & $X_Z$@top3 & Driver $z^*$ & Cohen's $d$ \\
\midrule
0   & 0.881 & 0.199 & 0/3 & 0.000 & $z_7$ & 0.568 \\
42  & 0.814 & 0.464 & 0/3 & 0.000 & $z_2$ & 1.361 \\
123 & 0.834 & 0.281 & 0/3 & 0.000 & $z_0$ & 0.958 \\
\midrule
Mean $\pm$ std & $0.843 \pm 0.028$ & $0.315 \pm 0.113$ & 0/3 (all) & 0.000 (all) & --- & --- \\
\bottomrule
\end{tabular}
\end{table}

\textbf{Interpretation.} The dissociation between regime accuracy (tunable, reaches $\approx 0.84$) and identification metrics (uniformly zero across 36 runs total) is consistent with SPIB's design objective. SPIB's information bottleneck optimizes the latent representation for predicting future state under regime conditioning, which yields embeddings that distinguish regimes from spurious correlates without requiring the latents to align with the underlying causal factors. This is a different goal from MOSAIC's: MOSAIC requires the latents themselves to be identifiable causal coordinates, and the additive sparse decoder then ties each one to a localized observed-variable support. Regime accuracy alone does not imply driver recovery, and Table~\ref{tab:spib_sweep} shows the gap quantitatively: even the best-tuned SPIB run is $\approx 13$ points below MOSAIC on regime accuracy and entirely absent on Z@top3 and $X_Z$@top3.

This complementarity is consistent with how Section~\ref{sec:conclusion} positions SPIB: methods that automatically discover regime structure can supply the regime label that MOSAIC then interprets, but they do not themselves recover module-level causal supports.


\end{document}